\documentclass[runningheads]{llncs}

\usepackage{eccv}

\usepackage{eccvabbrv}

\usepackage{graphicx}
\usepackage{booktabs}

\usepackage[accsupp]{axessibility}  %

\usepackage{hyperref}

\usepackage{latexsym,bm,amsmath}
\usepackage{array}
\usepackage{multirow}
\usepackage{wrapfig}
\usepackage[table]{xcolor}
\definecolor{cvprblue}{rgb}{0.21,0.49,0.74}

\usepackage{orcidlink}

\newcolumntype{C}[1]{>{\centering\arraybackslash}m{#1}}
\newcolumntype{L}[1]{>{\raggedright\arraybackslash}m{#1}}
\newcommand{\modelname}{EmbodiedVAE}

\begin{document}

\title{EmbodiedVAE:\\ Disentangled Video VAE for Efficient and Controllable Embodied Manipulation} 

\titlerunning{Abbreviated paper title}

\author{Jiayi Luo\inst{1,2} \and
Hanxin Zhu\inst{3,}\textsuperscript{$\ast$} \and 
Chen Gao\inst{1,4} \and \\
Jiankun Wang\inst{1} \and 
Cong Wang\inst{5,2} \and
Tianyu He\inst{6} \and \\
Jianxin Li\inst{1,}\textsuperscript{$\dagger$} \and
Zhibo Chen\inst{3,2,}\textsuperscript{$\dagger$}
}

\authorrunning{F.~Author et al.}

\institute{
Beihang University \and
Zhongguancun Academy \and
University of Science and Technology of China \and
National University of Singapore \and
CASIA, Institute of Automation, Chinese Academy of Sciences \and
Microsoft Research Asia \\
\textsuperscript{$\ast$}Project Lead\quad
\textsuperscript{$\dagger$}Corresponding Authors
}

\maketitle

\begin{abstract}
Latent diffusion models (LDMs) have recently significantly advanced embodied learning in constructing powerful embodied manipulation world models. 
However, despite the remarkable performance, existing LDMs predominantly rely on Variational Autoencoders (VAEs) optimized for natural scenes while failing to account for the unique characteristics of embodied manipulation scenarios, yielding latent representations that are neither compact nor controllable, thereby hindering efficient training of LDMs and precise robotic control. 
To solve this problem, we present \textbf{\underline{\modelname}}, a novel video VAE that provides \textbf{compact yet controllable} latent representations tailored for the robotic manipulation world models. Specifically, \modelname{} adopts a dual-encoder, single-decoder architecture with an asymmetric spatio-temporal compression module, which automatically disentangles the robot arm's motion from background environment, resulting in overall compactness while providing explicit embodied latent to support fine-grained action control.
To further preserve the temporal consistency of learned robotic motion latent, we introduce an optimal-transport-based consistency module that explicitly enforces motion fidelity and inter-frame coherence. 
Extensive experiments demonstrate that our proposed \modelname{} achieves superior reconstruction quality with high compression rate, while enabling more precise action control in robotic manipulation scenarios with an average of 2dB PSNR improvement over state-of-the-art video VAEs.

Code is available at: \url{https://github.com/Mutual-Luo/EmbodiedVAE}
\end{abstract}

\section{Introduction}
\label{sec:intro}

Recent advances in latent diffusion models (LDMs) have demonstrated remarkable potential for embodied learning, particularly in constructing expressive embodied manipulation world models~\cite{zhu2025irasim, jiang2025enerverse,gao2026octonav}, which enable embodied agents to imagine future interactions with environments at a low cost.
At the heart of LDMs lies a variational autoencoder (VAE) that encodes high-dimensional data into latent representations through an encoder-decoder architecture~\cite{rombach2022high, van2017neural, vahdat2020nvae}.
Within this latent space, the diffusion process operates to progressively refine these representations, enabling high-fidelity visual generation~\cite{wan2025wan, yang2024cogvideox, blattmann2023align, blattmann2023stable}.

While the remarkable performance of LDMs, directly applying them to build embodied robotic manipulation world models still faces some challenges.
\begin{wrapfigure}{r}{0.64\textwidth}
\centering
\includegraphics[width=\linewidth]{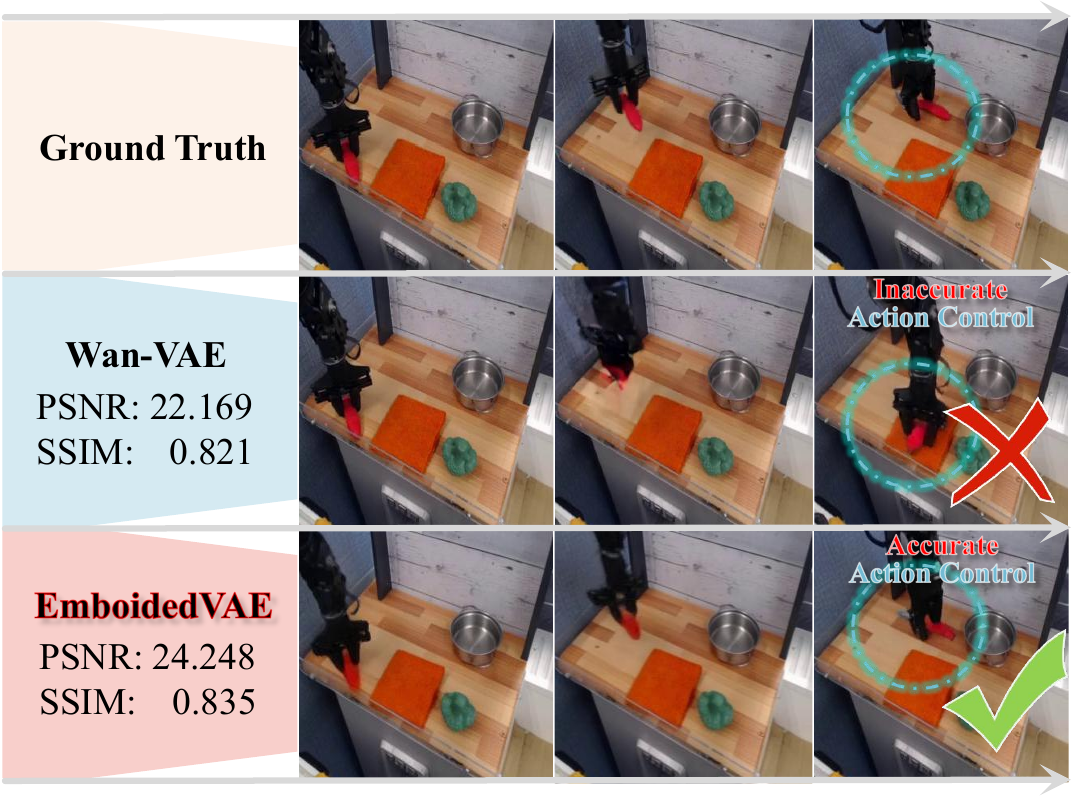}
\caption{
Example comparison on the Bridge dataset~\cite{walke2023bridgedata} illustrating latent controllability of the video VAE in robotic manipulation scenarios, given history frames and future arm actions. 
The LDM backbone is fixed as IRASim-L (461 M)~\cite{zhu2025irasim}, while \textit{the only difference lies in the video VAE used for training and inference}.
}
\vspace{-1.0em}
\label{fig:intro}
\end{wrapfigure}
In such scenarios, the core objective is to accurately predict the outcomes of arm-environment interactions under given future action conditions~\cite{chen2025robotwin, mu2025robotwin, bu2025agibot, makoviychuk2021isaac, kim2024openvla}. 
However, existing LDMs predominantly rely on image or video VAEs designed for natural scenes. As a result, their latents remain insufficiently compact and controllable, thereby \textit{hindering efficient downstream training and precise robotic control}.
While frame-by-frame compression with 2D image VAEs preserves temporal continuity, it fails to exploit inter-frame redundancy, producing large latents that hinder efficient LDM training~\cite{tang2024vidtok, yu2024efficient, luo2026attention}.
Conversely, although existing video VAEs achieve relatively compact latents via spatial-temporal compression, their temporal compression, which lacks explicit modeling of the robotic actions, often causes the motion semantics to be lost.~\cite{wang2025vidtwin, zhu2025irasim, jiang2025enerverse}.
As shown in Figure~\ref{fig:intro}, given history frames and corresponding future actions, the world model trained upon the widely used Wan-VAE~\cite{wan2025wan} yields high-quality visuals but fails to enable accurate action control.
In summary, existing video VAEs in embodied scenarios still struggle to balance compression and motion fidelity, limiting their capacity for accurate motion control in downstream video generation.

To address the aforementioned problems, in this work, we propose our novel \textbf{\underline{\modelname}}, a video VAE tailored for embodied robotic manipulation world models that produces compact yet controllable latent representations to enable efficient and effective downstream embodied learning.
The core idea of our proposed \modelname{} is to automatically disentangle an input video into two compact latent spaces: one capturing the foreground fine-grained, robot-related motion information, and the other modeling the contextual background. 
This explicit disentangle allows the robotic latent to directly incorporate action conditions on the manipulator itself while remaining unaffected by irrelevant environmental noise, leading to more accurate and controllable action generation.

Specifically, \modelname{} disentangles the robotic arm and environmental context with a dual-encoder, single-decoder architecture. The environment encoder captures scene context, while the robotic arm encoder models the manipulator’s motion dynamics.
\textit{(1) To learn to attend to different types of information}, we employ a two-stage training scheme where robotic arm masks are utilized solely during training to facilitate the disentanglement of motion and context. After training, \modelname{} operates without any mask input.
In the first stage, two VAEs are trained for robotics and environment, and in the second, their frozen encoders guide a unified decoder to reconstruct complete videos from disentangled latents.
\textit{(2) To obtain highly compact latent representations}, we adopt an asymmetric compression strategy: the robotic arm encoder applies stronger spatial compression to focus on fine-grained motion, while the background encoder applies stronger temporal compression to exploit redundancy, yielding compact yet informative latent representations.
\textit{(3) To capture motion consistency in the robotic arm latent}, we introduce an optimal-transport-based loss that minimizes information transport cost across frames, encouraging stable motion dynamics and temporal coherence in the learned representations.

Our main contributions can be summarized as follows:
\begin{itemize}
    \item  We propose \modelname{}, a novel video VAE tailored for embodied robotic manipulation that provides compact yet controllable latent representations. 
    \item We achieve disentangled and compact latents using a dual-encoder, single-decoder architecture trained in two stages with asymmetric compression rates. An optimal-transport-based loss further enforces temporal consistency of robotic arm motion in the latent space.
    \item Extensive experiments on embodied manipulation tasks show that \modelname{} achieves high-quality reconstruction under strong compression, enabling more accurate and consistent action control in the downstream tasks.
\end{itemize}

\section{Related Work}
\subsection{Visual Autoencoders}
Recent advances in visual autoencoders have greatly improved image and video generation by learning compact, expressive latent representations~\cite{wu2025improved}. They typically encode visual data into a low-dimensional latent space, as in the 2D VAE of Stable Diffusion~\cite{rombach2022high,blattmann2023stable}, while early video extensions inflated 2D convolutions to 3D~\cite{blattmann2023stable}.
Subsequent works developed fully spatio-temporal autoencoders with explicit temporal compression~\cite{yang2024cogvideox,qin2024xgen,wan2025wan}.
These methods can be broadly categorized into discrete and continuous autoencoders: discrete variants quantize features into a finite codebook of tokens via vector quantization~\cite{van2017neural,yan2021videogpt,chang2022maskgit,yu2023language}, while continuous variants learn a Gaussian latent space through differentiable sampling for smooth, end-to-end visual reconstruction and generation~\cite{agarwal2025cosmos,yang2024cogvideox,yang2024cogvideox,wan2025wan,tang2024vidtok,podell2023sdxl}.

\begin{figure*}[!t]
\centering
\includegraphics[width=\linewidth]{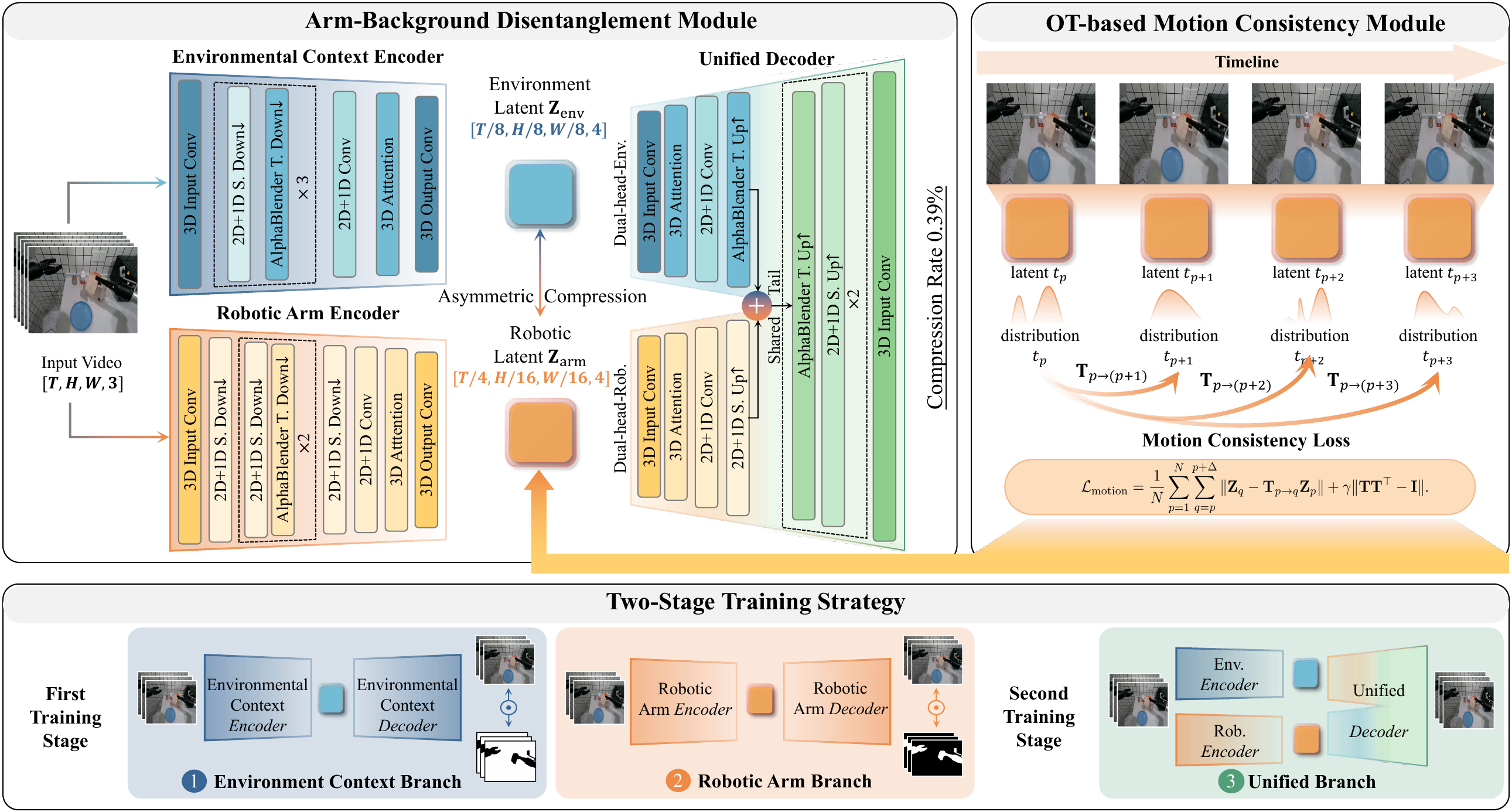}
\caption{The architecture of our \modelname{}, a novel video VAE tailed for embodied manipulation. The Arm–Background Disentanglement Module adopts a dual-encoder, single-decoder architecture with asymmetric spatio-temporal compression, producing compact yet controllable latent representations. To ensure temporal motion coherence, the OT-based Motion Consistency Module further regularizes the robotic arm's latent. The entire \modelname{} is optimized through a two-stage training strategy: the first stage learns the two encoders, while the second stage trains the unified decoder.}
\label{fig:framework}
\end{figure*}

\subsection{High-Compression Video VAEs}
Operating in a compact latent space allows LDMs to train and sample more efficiently, which requires a video autoencoder capable of strong compression while preserving essential semantics.
Several recent works have explored this direction~\cite{tian2024reducio,chen2024deep,liu2025hi}. LTX-Video~\cite{hacohen2024ltx} adopts a video VAE with an aggressive downsampling ratio, whose decoder also performs the final diffusion denoising step. VidTwin~\cite{wang2025vidtwin} achieves a 0.20\% compression rate by decoupling videos into structural and dynamic latent vectors, and Hi-VAE\cite{liu2025hi} further decomposes video dynamics into global motion and fine-grained spatial details. H3AE~\cite{wu2025h3ae} investigates efficient architectures and introduces a latent loss that eliminates the need for complex discriminators.
\textit{However, existing video VAEs still struggle to balance compactness and controllability, limiting their ability to support precise action control in downstream robotic manipulation video generation.}

\subsection{Robotic Manipulation Video Generation}
Video generation models used as world simulators have advanced embodied learning by predicting future states in videos that capture realistic robot-object interactions conditioned on past observations and actions.
iVideoGPT~\cite{wu2024ivideogpt} performs autoregressive, action-conditioned video prediction for embodied tasks, while VLP~\cite{du2023video} and UniSim~\cite{yang2023learning} integrate language and action cues to enhance semantic controllability in video generation.
EVAC~\cite{jiang2025enerverse} introduces a multi-level action-conditioning mechanism with ray-map encoding to improve generalization across environments.
IRASim~\cite{zhu2025irasim} further incorporates a frame-level action-conditioning module within each transformer block to explicitly align actions with generated frames and strengthen temporal consistency.
\textit{However, most existing methods rely on image VAE latents, which limit the training efficiency, while video VAEs struggle to provide controllable latents.}

\section{\modelname}
In this section, we elaborate on \modelname{}, a high-compression video VAE that offers controllable latent representations tailored for embodied manipulation.
As illustrated in Figure~\ref{fig:framework}, our \modelname{} is built upon a dual-encoder, single-decoder architecture with asymmetric compression rates. 
One encoder extracts foreground robotic arm features, while the other captures environmental context.
A unified decoder then reconstructs the entire scene from the fused latent representations, optimized under our designed two-stage training strategy. 
To further enhance motion stability, we introduce an optimal-transport-based consistency loss that enforces temporal coherence in the latent space.

\subsection{Arm-Background Disentanglement Module}
Accurate action control is essential for robotic manipulation video generation, such as in robotic world models, where the system must precisely simulate and predict complex interactions among the robot, manipulated objects, and the surrounding environment~\cite{zhu2025irasim,jiang2025enerverse,gao2023room}. 
However, most existing world models based on LDMs still rely on 2D image VAEs for latent representations~\cite{zhu2025irasim,blattmann2023stable}, since temporal compression in video VAEs often leads to motion information loss, making precise action control difficult.
Although the 2D image VAEs effectively capture spatial structures, they fails to fully exploit temporal redundancy. As a result, the latent space remains insufficiently compact, leading to inefficient training and limited scalability for downstream LDMs.

To address the aforementioned challenges, our proposed \modelname{} aims to learn a dedicated latent space for robotic arm motion, enabling the precise integration of downstream action-control information into the robotic arm representation while suppressing irrelevant backgrounds. 
To achieve this, \modelname{} employs a dual-encoder, single-decoder architecture that automatically disentangles the robotic arm from its environmental surroundings without requiring explicit masks during inference. 
One encoder focuses on capturing fine-grained robotic arm dynamics, while the other encoder models the environmental context. 
This dual-encoder design produces two complementary latent representations, the robotic arm and the environment latent. The explicit robotic latent enables precise action injection while remaining unaffected by unrelated noise.

\subsubsection{Foreground Robotic Arm Encoder}
The detailed network architecture is visualized in Figure~\ref{fig:framework}.  Following VidTok~\cite{tang2024vidtok}, we avoid using fully 3D architectures to reduce the computational overhead, while we only apply causal 3D convolutions in both the input and output layers to integrate spatio-temporal information and preserve temporal causality.
In the penultimate layer, a 3D attention block is introduced to capture fine-grained spatio-temporal dependencies and enhance feature representation. 
Between the input and output stages, we stack multiple hybrid blocks composed of 2D spatial and causal 1D temporal downsampling modules, fused by the AlphaBlender operator~\cite{lin2024open}:
\begin{equation}\label{eq:alphablender}
    \textbf{x}_{\text{down}} = \alpha \cdot \textbf{x}_{1} + (1-\alpha)\cdot \textbf{x}_{2},
\end{equation}
where $\textbf{x}_{1}$ and $\textbf{x}_{2}$ represent the intermediate downsampled features produced from the interpolation operation and the convolution operation branch, respectively. The blending weight $\alpha$ is a fixed hypermeter in practice, and we set $\alpha=\text{Sigmoid}(0.2)$ to balance the information flow from two branches here, and the $\textbf{x}_{\text{down}}$ serves as the final downsampled representation as defined.

Since the robotic arm in embodied manipulation scenarios typically occupies only a small, albeit crucial, fraction of the visual field, our foreground robotic-arm encoder therefore applies a higher spatial compression rate to its feature representation to better match its limited spatial footprint, resulting in a more compact latent representation. Specifically, we stack four spatial downsampling blocks together with two temporal downsampling blocks, achieving an overall compression ratio of 4$\times$(16$\times$16) (temporal $\times$ spatial). We set the number of latent channels to 4, which further reduces spatial redundancy while preserving the most informative arm-centric details, yielding a compact and efficient robotic-arm latent representation for downstream modeling and control.

\subsubsection{Background Environmental Context Encoder}
For the background environmental context encoder, we adopt an architecture similar to the foreground robotic arm encoder. However, in embodied manipulation scenarios, the first-person perspective introduces substantial temporal redundancy in the background, which we leverage by applying a more aggressive temporal compression strategy. 
Specifically, we stack three spatial with three temporal downsampling blocks to acheive a 8$\times$(8$\times$8) (temporal $\times$ spatial) compression and we also set the latent channels to 4,  effectively reducing the temporal redundancy and yielding a temporal compact environment latent representation.

\subsubsection{Unified Reconstruction Decoder}\label{sec:decoder}

As illustrated in Figure~\ref{fig:framework}, the decoder consists of two components: a dual-head module and a shared tail module. In the dual-head module, the robotic arm and environment latent representations are upsampled independently. 
The arm head performs stronger spatial upsampling, whereas the background head emphasizes temporal upsampling.
When both latent representations reach the same resolution, we concatenate them along the channel dimension and feed the merged latent into the shared tail decoder to reconstruct the complete video.
The whole workflow can be formulated as:
\begin{gather}
    \mathbf{\hat{X}} = D(\mathbf{Z}_{\text{arm}}, \mathbf{Z}_{\text{env}}), \\
    \mathbf{Z}_{\text{arm}} = R(E_{\text{arm}}(\mathbf{X})), \quad 
    \mathbf{Z}_{\text{env}} = R(E_{\text{env}}(\mathbf{X})).
\end{gather}
Here, $\mathbf{\hat{X}}$ and $\mathbf{X}$ denote the reconstructed and input videos, respectively, and $\mathbf{Z}_{arm}, \mathbf{Z}_{env}$ denotes the two latent representations. $D(\cdot), E_{\text{arm}}(\cdot)$ and $E_{\text{env}}(\cdot)$ represent the decoder and the two encoders, while $R(\cdot)$ denotes the latent regularizer.
To mirror the architectures of the robotic-arm encoder and the environmental-context encoder, the unified decoder stacks multiple 2D and 1D upsampling blocks integrated with the AlphaBlender operator defined in Eq.~\eqref{eq:alphablender}; the input and output layers use 3D convolutions, and the second layer incorporates a 3D attention block to ensure architectural completeness and consistency of VAE.

\subsection{OT-based Motion Consistency Module}
The motion of the robotic arm carries the essential information in embodied manipulation scenarios, reflecting the interactions between the embodied agent and its environment in a fine-grained, time-varying manner.
However, the temporal compression used in existing video VAEs often causes motion details to be lost during the latent encoding and reconstruction processes. This loss of motion information in the latent space impairs downstream robotic arm action control that relies on these latents for accurate, consistent decision making.

To maximize the preservation of motion information in the resulting compact latents, we introduce an Optimal-Transport-based Motion Consistency Module.
Our design is inspired by the observation that stable motion patterns in videos maintain visual invariance~\cite{wang2019learning}, meaning that visual correspondences between patches remain coherent across consecutive frames.
Such temporal visual invariance reflects the intrinsic stability of motion dynamics. Our module leverages this property to encourage the latent representations to remain motion consistent over time by minimizing the information transport cost across latent frames. 
This constraint is applied to the latent space of the robotic arm encoder, guiding it to stabilize capture temporally coherent motion dynamics, since its corresponding decoder is discarded after training as introduced in  Section~\ref{sec:training}.

Optimal Transport (OT) is a mathematical framework to measure how much effort it takes to transform one distribution into another~\cite{santambrogio2015optimal}, and has proven effective in vision tasks such as semantic correspondence~\cite{liu2020semantic} and visual place recognition~\cite{izquierdo2024optimal} by facilitating coherent alignment of feature distributions.
Specifically, let the latent distributions at timestamp $t_{p}$ and $t_{q}$ be $\mu_{p}, \mu_{q}$ over the latent space $\mathcal{Z}$. The Kantorovich optimal transport plan~\cite{santambrogio2015optimal} between distribution $\mu_{p}$ and distribution $\mu_{q}$ is a coupling $\pi_{p\rightarrow q}\in\Pi(\mu_{p},\mu_{q})$ that minimizes the following expected transport cost under the standard OT formulation:
\begin{equation}\label{eq:ot_definition}
    \pi^{\star}=\arg\min_{\pi\in\Pi(\mu_{p},\mu_{q})}\int c(z_{p},z_{q})d\pi(z_{p},z_{q}),
\end{equation}
where  $c(\cdot, \cdot)$ denotes the transport cost and $\Pi(\mu_{p},\mu_{q})$ is the joint distribution.
Due to the discrete nature of video pixels, here we model the latent distribution as discrete probability distributions: $\mathcal{\mu}_{p}=\sum_{i=1}^{N}a_{i}\delta_{x_{i}}$ and $\mathcal{\mu}_{q}=\sum_{j=1}^{N}b_{j}\delta_{y_{j}}$, where $\delta_{x_{i}}$ denotes the Dirac measure~\cite{erbar2020computation} at $x_{i}$, satisfying $\int f(x)\delta_{x_{i}}(x)dx=f(x_{i})$, and $\mathcal{A}=\{a_{i}\}_{i}^{N}, \mathcal{B}=\{b_{j}\}_{j}^{N}$ represent the corresponding probability masses. 
As the similarity between two latent representations roughly encodes the likelihood of content correspondence across time, consequently, we design the cost function $c(\cdot,\cdot)$ as the negative similarity between these two latent slices:
\begin{equation}
    c(i,j) = -\mathbf{S}_{i,j}=\frac{\left \langle \mathbf{Z}_{p}[i],\mathbf{Z}_{q}[j] \right \rangle }{\sqrt{d}}
\end{equation}
where $\mathbf{Z}_{p}[i]$ denotes the latent feature at position $i$ and $d$ is the channel dimension.
Building upon this, and following~\cite{cuturi2013sinkhorn}, we can reformulate the optimal transport problem in Eq.~\eqref{eq:ot_definition} into an entropy-regularized formulation:
\begin{equation}\label{eq:ot_discrete}
\begin{aligned}
    \mathbf{T}^{\star} &=\arg\min_{\mathbf{T}} \sum_{ij}\mathbf{T}_{ij} c(i,j)+\epsilon \mathcal{H}(\mathbf{T}), \\
    & \text{s.t.}\quad \mathbf{T}\mathbf{1}=\mu_{p}, \mathbf{T}^{\top}\mathbf{1}=\mu_{q}
\end{aligned}
\end{equation}
where $\mathcal{H}(\mathbf{T}) = -\sum_{ij} \mathbf{T}_{ij} (\log \mathbf{T}_{ij} - 1)$ denotes the entropy regularizer, and $\varepsilon$ controls the strength of regularization during optimization. The matrix $\mathbf{T}$ represents the discrete optimal transport plan that minimizes the information transport cost between two latent distributions, serving as the discrete counterpart of $\pi^{\star}$ in this setting. We compute $\mathbf{T}$ using the Sinkhorn-Knopp algorithm~\cite{cuturi2013sinkhorn}, a differentiable and efficient approximation to the Hungarian algorithm~\cite{munkres1957algorithms} in practice.
After obtaining $\mathbf{T}$, we define the motion consistency loss to enforce temporal stability and visual invariance under smooth motion as follows:
\begin{equation}
    \mathcal{L}_{\text{motion}}=\frac{1}{N}\sum_{p=1}^{N}\sum_{q=p}^{p+\Delta}\Vert\mathbf{Z}_{q}-\mathbf{T}_{p \rightarrow q}\mathbf{Z}_{p} \Vert + \gamma \Vert \mathbf{T}\mathbf{T}^{\top} -\mathbf{I} \Vert.
\end{equation}
Here, $\Delta$ controls the temporal window for alignment, and $\gamma$ balances the permutation regularizer $\Vert \mathbf{T}\mathbf{T}^{\top}-\mathbf{I}\Vert$, which regularize $\mathbf{T}$ to approximate a one-to-one correspondence as closely as possible.
$\mathcal{L}_{\text{motion}}$ promotes temporal consistency by minimizing feature misalignment guided by the optimal transport plan computed above.
To further analyze the behavior of $\mathcal{L}_{\text{motion}}$, we present the following proposition, with details provided in Appendix~\ref{sec:proof} for completeness.
\begin{proposition}\label{prop1}
    Let $\mathbf{\hat{T}}^{\star}$ denote the OT plan obtained by the Sinkhorn–Knopp algorithm after row normalization. Denote  $j_i^{\star}=\arg\max_{j}\mathbf{S}_{ij}$ and define the margin $\gamma_i=\mathbf{S}_{i j_i^{\star}}-\max_{j\neq j_i^{\star}}\mathbf{S}_{ij}$, we have $\mathbf{\hat{T}}^{\star}_{ij^{\star}}\ge\frac{1}{1+ae^{-\gamma_{i}/\epsilon}}$ with a constant $a$.
\end{proposition}
Proposition \ref{prop1} implies that when latent features exhibit clear separability across frames and the motion remains temporally smooth, the transport matrix $\mathbf{T}$ will converge toward a nearly permutation-like structure, reflecting a strong motion temporal consistency. 
In contrast to a similarity matrix that encodes only \textit{local} pairwise affinities, the optimal transport plan $\mathbf{T}$ establishes a \textit{global} alignment between distributions, capturing coherent many-to-many correspondences that account for all pairwise relations across time~\cite{cuturi2013sinkhorn,izquierdo2024optimal} more effectively.

\subsection{Two-stage Training Strategy}\label{sec:training}
The specific training strategy of our proposed novel \modelname{} consists of two training stages. In the first training stage, the two encoders are obtained, while in the second training stage, we train the unified decoder.

\subsubsection{First training stage.}\label{sec:first_training_stage}
In the first stage, our objective is to train the robotic arm encoder to capture fine-grained motion of the manipulator, while the environmental encoder preserves scene context.
To achieve this, we first generate robotic arm masks using SAM2~\cite{ravi2024sam} guided by Grounding DINO~\cite{ren2024grounding} with the text prompt ``robotic arm''. It is worth noting that these masks are used only as auxiliary supervision during training; once trained, \modelname{} can automatically disentangle motion and context information without relying on masks. The training target loss for the robotic arm is defined as:
\begin{equation}\label{eq:loss_arm}
    \mathcal{L}_{\text{arm}} = \mathcal{L}_{\text{rec}} +\alpha \mathcal{L}_{\text{motion}} + \beta \mathcal{L}_{\text{KL}} + \lambda\mathcal{L}_{\text{aux}},
\end{equation}
where the reconstruction loss $\mathcal{L}_{\text{rec}}=\Vert (\mathbf{\hat{X}}-\mathbf{X})\odot \mathbf{M}) \Vert $, and $\mathbf{M}\in\{0,1\}^{T\times H\times W}$ is the pre-built binary robotic arm mask. 
$\mathcal{L}_{\text{KL}}$ denotes the KL divergence loss term that defined as $\mathcal{L}_{\text{KL}}=\text{KL}(\mathcal{N}(\mathbf{u}_\mathbf{z},\boldsymbol{\sigma}_{\mathbf{z}})\Vert \mathcal{N}(\mathbf{0},\mathbf{I}))$, 
and $\mathcal{L}_{\text{aux}}$ represents the auxiliary loss term which we further combine from the standard adversarial loss and the feature-level perceptual loss~\cite{tang2024vidtok}. 
$\alpha, \beta, \lambda$ is the hyperparameters.
As for the training of the environmental context encoder, the overall loss is defined as:
\begin{equation}
    \mathcal{L}_{\text{env}}= \mathcal{L}_{\text{rec}} + \beta \mathcal{L}_{\text{KL}} + \lambda\mathcal{L}_{\text{aux}},
\end{equation}
where the reconstruction loss $\mathcal{L}_{\text{rec}}=\Vert (\mathbf{\hat{X}}-\mathbf{X})\odot (1-\mathbf{M}) \Vert $ and the other loss terms are the same as those used for training the robotic arm encoder in Eq.~\eqref{eq:loss_arm}.

\subsubsection{Second training stage}
After the first training stage, we freeze the two pre-trained encoders from the robotic-arm and environment-context branches with a very low learning rate while discarding their decoders. We then train the unified decoder, which is introduced in Section~\ref{sec:decoder}, with the following loss:
\begin{equation}
    \mathcal{L}_{\text{unified}}= \mathcal{L}_{\text{rec}} + \beta \mathcal{L}_{\text{KL}} + \lambda\mathcal{L}_{\text{aux}},
\end{equation}
where the reconstruction loss term$\mathcal{L}_{\text{rec}}=\Vert \mathbf{\hat{X}}-\mathbf{X} \Vert $, and the remaining terms are identical to the first training stage introduced in Section~\ref{sec:first_training_stage}.

\subsection{Action-controlled Video Generation}
Our \modelname{} can be seamlessly integrated with most existing LDMs, enabling the efficient and controllable downstream robotic action-controlled video generation. 
To demonstrate this integration, we adopt IRASim~\cite{zhu2025irasim}, a DiT-style LDM that functions as the robotic manipulation world model, as an example in our experiments.
In this setup, the latent representations of the robotic arm and the environmental context, $\mathbf{Z}_{\text{arm}}$ and $\textbf{Z}_{\text{env}}$, are first converted into two sequences of tokens with the same hidden dimension.
Both latent representations are processed through the same diffusion procedure using shared DiT parameters, \textit{ensuring that no additional computational overhead is introduced.} To facilitate effective interaction between the two complementary latent streams, we insert a cross-attention layer after every $k$ transformer blocks, where $k=2$.
This design enables adaptive information exchange between the arm and background representations. For example, the arm motion can be adjusted based on nearby objects, while the environment representation accounts for motion-induced changes such as lighting and shadows. This interaction results in globally consistent scene dynamics and locally accurate motion generation.

\begin{table*}[!b]
  \centering
  \caption{Quality comparison with eight baselines on the \textbf{video VAE reconstruction task}, evaluated on two datasets using four metrics. The best results are highlighted in \textbf{bold}, and the second-best results are \underline{underlined}.}
  \vspace{-0.5em}
  \label{tb:reconstruction}
  \tabcolsep=0.0cm
  \resizebox{\linewidth}{!}{ 
    \begin{tabular}{C{2.8cm}|C{1.8cm}|C{1.3cm}C{1.3cm}C{1.3cm}C{1.5cm}|C{1.3cm}C{1.3cm}C{1.3cm}C{1.5cm}}
    \toprule
    \multicolumn{2}{c|}{Datasets} & \multicolumn{4}{c|}{Agibot-2025~\cite{jiang2025enerverse}} & \multicolumn{4}{c}{Bridge~\cite{walke2023bridgedata}} \\
    \midrule
    Method & Com. Rate$\downarrow$ & PSNR$\uparrow$ & LPIPS$\downarrow$ & SSIM$\uparrow$ & FVD$\downarrow$  & PSNR$\uparrow$ & LPIPS$\downarrow$ & SSIM$\uparrow$ & FVD$\downarrow$ \\
    \midrule
    OpenSoraPlan~\cite{lin2024open} & 1.04\% & 30.7553 & \underline{0.0834} & \underline{0.9257} & \textbf{290.2550} & 30.3755 & \underline{0.1098} & 0.8948 & \underline{340.5284} \\
    Cosmos~\cite{agarwal2025cosmos} & 2.08\% & 28.2329 & 0.1520 & 0.8778 & 889.7757 & 27.8812 & 0.1916 & 0.8386 & 715.6345\\
    iVideoGPT~\cite{wu2024ivideogpt} & 1.50\% & 29.2503 & 0.1060 & 0.9096 & 760.6489 & 28.3325 & 0.1447 & 0.8738 & 522.1818 \\
    MAGVIT-v2~\cite{yu2023language} & 0.65\% & 27.1014 & 0.2203 & 0.7197 & 849.5063 & 26.8635 & 0.2396 & 0.6986 & 549.9939 \\
    Vidtwin~\cite{wang2025vidtwin} & \textbf{0.20\%} & 30.8263 & 0.1216 & 0.9202 & 443.1976 & 29.7386 & 0.1679 & 0.8531 & 585.8065 \\
    EMU-3~\cite{wang2024emu3} & 0.53\% & 28.9081 & 0.1064 & 0.8926 & 599.8921 & 27.5814 & 0.1517 & 0.8317 & 736.5338 \\
    CV-VAE~\cite{zhao2024cv} & 0.53\% & 30.2767 & 0.1018 & 0.9128 & 392.6846 & 29.8419 & 0.1330 & 0.8789 & 418.8394 \\
    CMD~\cite{yu2024efficient}   & 6.85\% & \underline{31.4030} & 0.1086 & 0.8995 & 439.9010 & \textbf{31.3764} & 0.1115 & \underline{0.8996} & 358.6915 \\
    \midrule
    \textbf{\modelname{}} & \underline{0.39\%} & \textbf{31.6745} & \textbf{0.0723} & \textbf{0.9345} & \underline{368.5511} & \underline{30.6957} & \textbf{0.1060} & \textbf{0.9017} & \textbf{309.7604} \\
    \bottomrule
    \end{tabular}%
}
  \vspace{-1.5em}
\end{table*}

\section{Experiments}

\subsection{Experiments Setup}
\textbf{Datasets.}
We train our proposed \modelname{} on a self-collected dataset of approximately one million robotic manipulation videos, aggregated from RobNet~\cite{dasari2019robonet}, RobSet~\cite{bharadhwaj2024roboagent}, BC-Z~\cite{jang2022bc}, RH20T~\cite{fang2023rh20t}, and DROID~\cite{khazatsky2024droid}, covering diverse manipulation behaviors across multiple environments.
Robotic arm masks are automatically generated for each video using SAM2~\cite{ravi2024sam} and Grounding DINO~\cite{ren2024grounding} with the text prompt ``robotic arm''.
For evaluation, we use Agibot-2025~\cite{jiang2025enerverse} and Bridge~\cite{walke2023bridgedata} for video VAE reconstruction tasks, and RT-1~\cite{brohan2022rt} and Bridge~\cite{walke2023bridgedata} for action-conditioned robotic manipulation video generation.

\noindent\textbf{Baselines.}
(1) \textit{For reconstruction task evaluation}, we compare our \modelname{} with several state-of-the-art high-compression video VAEs that also achieve good reconstruction quality, including OpenSoraPlan~\cite{lin2024open}, MAGVIT-v2~\cite{yu2023language}, and CV-VAE~\cite{zhao2024cv} with continuous latents; EMU-3~\cite{wang2024emu3} and Cosmos~\cite{agarwal2025cosmos} as discrete token-based models; and approaches that decouple videos into multiple latent spaces, such as CMD~\cite{yu2024efficient}, iVideoGPT~\cite{wu2024ivideogpt}, and VidTwin~\cite{wang2025vidtwin}.
(2) \textit{For downstream video generation in action-controlled robotic manipulation}, we compare \modelname{} with several representative VAEs. Specifically, we include high-fidelity video VAEs such as Wan-VAE~\cite{wan2025wan} and Cog-VAE~\cite{yang2024cogvideox}; high-compression VAEs including OpenSoraPlan~\cite{lin2024open}, EMU-3~\cite{wang2024emu3}, and CV-VAE~\cite{zhao2024cv}; and models that decouple videos into multiple latent components, such as VidTwin~\cite{wang2025vidtwin} and CMD~\cite{yu2024efficient}. 
In addition, we evaluate the 2D image VAE SDXL~\cite{podell2023sdxl} from Stable Diffusion, a strong baseline that encodes videos frame by frame, retaining full temporal information without temporal compression by a non-compact latent.

\noindent\textbf{Metrics.}
Following~\cite{tang2024vidtok,zhu2025irasim}, we evaluate all methods on video reconstruction and action-conditioned video generation for robotic manipulation using PSNR~\cite{hore2010image}, SSIM~\cite{wang2004image}, LPIPS~\cite{zhang2018unreasonable}, and FVD~\cite{unterthiner2018towards}.
Efficiency is assessed by the VAE compression rate~\cite{tang2024vidtok,rombach2022high}, defined as the ratio of latent to original video dimensions.

\noindent\textbf{Implements.}
We train \modelname{} on 8 fps, 16-frame video clips at a resolution of 256 $\times$ 256, and evaluate it on 16 fps, 16-frame clips of the same resolution for reconstruction tasks. The model is trained using 8 NVIDIA A800 GPUs.

\subsection{VAE Reconstruction Quality}\label{sec:exp_recon}

\begin{wrapfigure}{l}{0.5\textwidth}
\centering
\vspace{-2.3em}
\includegraphics[width=\linewidth]{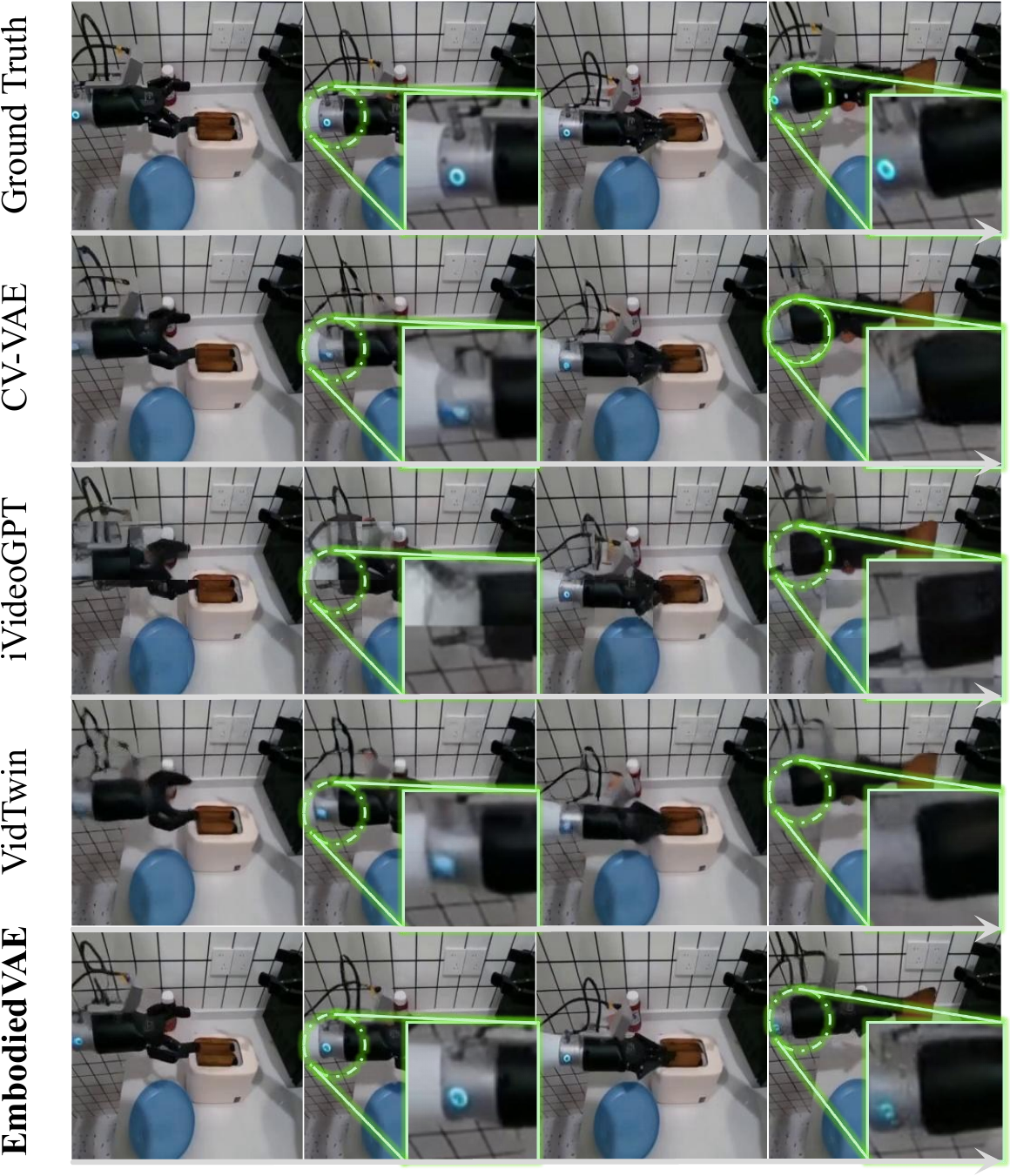}
\vspace{-1em}
\caption{Comparison of the \textbf{video VAE reconstruction quality} between the proposed \modelname{} and existing representative baseline methods. Results are shown from the Agibot-2025~\cite{jiang2025enerverse} dataset.}
\vspace{-2em}
\label{fig:reconstruction}
\end{wrapfigure}

In the video VAE reconstruction experiments, we comprehensively compare the proposed \modelname{} with all baseline methods in terms of both reconstruction quality and compression rate, and the results are summarized in the Table~\ref{tb:reconstruction}. 
As we can seen, the proposed \modelname{} achieves the best overall reconstruction performance across both datasets (except for PSNR on the Bridge dataset) while maintaining a competitive compression rate of 0.39\%. 
In particular, \modelname{} achieves the highest PSNR of 31.6745 in Agibot-2025, while its compression rate is nearly 1/20 of that of the runner-up baseline CMD, demonstrating superior reconstruction fidelity under high compression rate. 
We further present a qualitative example in Figure~\ref{fig:reconstruction}, where \modelname{} produces substantially clearer reconstructions of robotic arm details with a high compression rate compared with baselines.
In particular, baselines such as VidTwin fail to preserve fine-grained robotic arm visual cues like electrical components, whereas our \modelname{} successfully reconstructs these critical details.

\begin{table*}[!t]
  \centering
  \caption{Quality comparison with seven baselines on the \textbf{downstream action-controlled video generation in robotic manipulation scenarios} using four metrics. The best results are highlighted in \textbf{bold}, the second-best results are \underline{underlined}, and \colorbox[HTML]{E6E6E6}{value$^{\star}$} denotes results obtained with the \colorbox[HTML]{E6E6E6}{image VAE baseline} used during both training and inference, which is \textit{without any temporal compression}.}
  \vspace{-0.5em}
  \label{tb:generation}
  \tabcolsep=0.00cm
  \resizebox{\linewidth}{!}{ 
    \begin{tabular}{C{2.8cm}|C{2.0cm}|C{1.3cm}C{1.3cm}C{1.3cm}C{1.5cm}|C{1.3cm}C{1.3cm}C{1.3cm}C{1.5cm}}
    \toprule
    \multicolumn{2}{c|}{Datasets} & \multicolumn{4}{c|}{RT-1~\cite{brohan2022rt}}     & \multicolumn{4}{c}{Bridge~\cite{walke2023bridgedata}} \\
    \midrule
    Method & Com. Rate↓ & PSNR↑ & LPIPS↓ & SSIM↑ & FVD↓  & PSNR↑ & LPIPS↓ & SSIM↑ & FVD↓ \\
    \midrule
    \rowcolor[HTML]{E6E6E6}
    SDXL~\cite{podell2023sdxl} & 8.32\% &  \underline{23.1988}$^{\star}$   &   0.1901   &    0.8074   &   \underline{761.7837}$^{\star}$  &   \underline{22.3624}$^{\star}$   &   0.1753   &   0.7915  &  \textbf{455.7345$^{\star}$} \\
    \midrule
    OpenSoraPlan~\cite{lin2024open} & 1.04\% &    21.6494   &   0.2486    &    0.7432   &   958.0784    &   21.0405    &  0.2376     &   0.7577    & 785.5559 \\
    Emu3~\cite{wang2024emu3}  & 0.53\% &   21.6749    &    0.2456   &   0.7286    &   1049.2935    &   20.8573    &   0.2506    &    0.7053   & 986.3582 \\
    CV-VAE~\cite{zhao2024cv} & 0.53\% & 20.3296 & 0.3047 & 0.6859 & 2529.3193 & 19.3068 & 0.3544 & 0.6568 & 1829.7292 \\
    CMD~\cite{yu2024efficient}   & 6.85\% & 20.9454 & 0.2560 & 0.7671 & 2644.7208 & 20.1799 & 0.2521 & 0.7419 & 1464.6843 \\
    Vidtwin~\cite{wang2025vidtwin} & \textbf{0.20\%} & 21.1037 & 0.3710 & 0.6697 & 1824.3383 & 20.4466 & 0.3438 & 0.6889 & 1690.1818 \\
    Cog-VAE~\cite{yang2024cogvideox} & 2.08\% &   21.2231    &  0.2174     &    0.7907   &   1135.6917    & 21.6951 & 0.2095 & 0.7807 & 1302.4434 \\
    Wan-VAE~\cite{wan2025wan} & 2.45\% &   \underline{22.7334}    &    \underline{0.1848}   &   \underline{0.8138}    &    \underline{802.2071}   &   \underline{22.1690}    &   \underline{0.1601}    &   \underline{0.8211}    &  \underline{548.9621}\\
    \midrule
    \textbf{\modelname{}} & \underline{0.39\%} & \textbf{24.7082} & \textbf{0.1707} & \textbf{0.8263} & \textbf{716.3273} & \textbf{24.2484} & \textbf{0.1409} & \textbf{0.8366} & 631.6450 \\
    \bottomrule
    \end{tabular}%
}
\end{table*}

\begin{figure*}[!t]
\centering
\vspace{-0.8em}
\includegraphics[width=\linewidth]{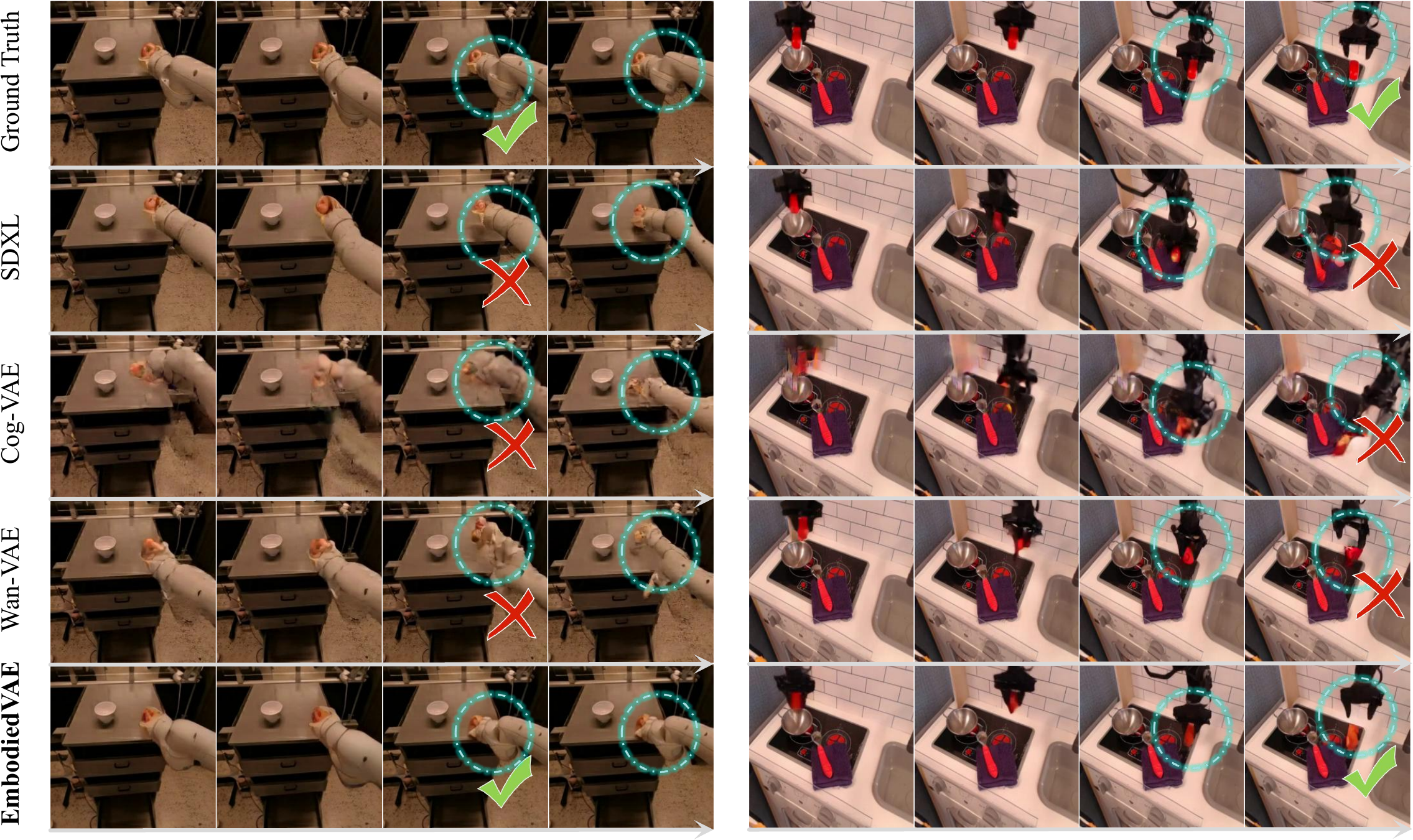}
\caption{Qualitative comparisons of \textbf{action-conditioned world models for robotic manipulation} trained under the same setup, where the only difference lies in the VAE used for latent representations: our \modelname{} versus representative video VAEs. Results on the left are from RT-1~\cite{brohan2022rt}, and results on the right are from Bridge~\cite{walke2023bridgedata}.}
\vspace{-1.3em}
\label{fig:generation}
\end{figure*}

\subsection{\texorpdfstring{\textit{\underline{Controllable Studies}}: Robotic Manipulation World Model}{Action-Controlled world model for Robotic Manipulation}}\label{sec:exp_control}
We evaluate \modelname{} on action-conditioned video prediction for robotic manipulation using IRASim~\cite{zhu2025irasim}, with the LDM backbone fixed to IRASim-L (461M) and only the VAE component varied during training and inference (Details are in Appendix~\ref{sec:exp}).
As shown in Table~\ref{tb:generation}, our \modelname{} significantly outperforms the widely recognized Wan-VAE and Cog-VAE on this task, achieving a notable improvement of 2.02 PSNR over the runner-up video VAE (Wan-VAE). 
It is worth noting that \modelname{} even surpasses the SDXL~\cite{podell2023sdxl}, a powerful image-based VAE without temporal information compression.
Example analysis in Figure~\ref{fig:generation} further show that \modelname{} can accurately predict robot-environment interactions, while competing video VAEs (e.g., Wan-VAE) often produce visually plausible yet physically inconsistent motions.
These advantages are attributed to our disentangled design, which enables precise action injection without interference from irrelevant factors, and the consistency loss further preserves the motion dynamics despite time compression.

\begin{wraptable}{r}{0.7\columnwidth}
  \centering
  \vspace{-3.5em}
  \caption{Ablation studies on both the video VAE reconstruction (``Recon.'') task and the downstream action-controlled video prediction for robotic manipulation task (``Manip.'').}
  \label{tb:ablation}
  \tabcolsep=0.1cm
  \resizebox{\linewidth}{!}{ 
   \begin{tabular}{c|c|cc|cc}
    \toprule
    \textbf{Task}  & \textbf{Methods} & \textbf{PSNR$\uparrow$} & \textbf{LPIPS$\downarrow$} & \textbf{PSNR$\uparrow$} & \textbf{LPIPS$\downarrow$} \\
    \midrule
    \multirow{4}[4]{*}{Recon.} & Datasets & \multicolumn{2}{c|}{Agibot-2025~\cite{jiang2025enerverse}} & \multicolumn{2}{c}{Bridge~\cite{walke2023bridgedata}} \\
\cmidrule{2-6}          & \modelname{} & 31.6745 & 0.0723 & 30.6957 & 0.1060 \\
          & \textit{w/o} \textit{Mask} & \textit{31.1812} & \textit{0.0728} & \textit{29.6680} & \textit{0.1211} \\
          & \textit{w/o} $\mathcal{L}_{\text{motion}}$ & \textit{31.0213} & \textit{0.0835} & \textit{29.9769} & \textit{0.1087}  \\
    \midrule
    \multirow{4}[4]{*}{	Manip.} & Datasets & \multicolumn{2}{c|}{RT-1~\cite{brohan2022rt}} & \multicolumn{2}{c}{Bridge~\cite{walke2023bridgedata}} \\
\cmidrule{2-6}          & \modelname{} & 24.7082 & 0.1707 & 24.2484 & 0.1409 \\
          & \textit{w/o} \textit{Mask} & \textit{23.8044} & \textit{0.1904} & \textit{23.7306} & \textit{0.1705} \\
          & \textit{w/o} $\mathcal{L}_{\text{motion}}$ & \textit{23.8187} & \textit{0.1863} & \textit{23.9726} & \textit{0.1643} \\
    \bottomrule
    \end{tabular}%
}
\vspace{-2em}
\end{wraptable}

\subsection{Ablation Studies}\label{sec:exp_ablation}
We conduct an ablation study to evaluate the effectiveness of the proposed OT-based motion consistency loss $\mathcal{L}_{\text{motion}}$ in both the video VAE reconstruction task and the downstream action-controlled video prediction task.
As shown in Table~\ref{tb:ablation}, this module consistently improves both tasks, with larger gains in the downstream manipulation setting.

\subsection{Disentangle Studies}
\begin{wrapfigure}{r}{0.75\textwidth}
\centering
\vspace{-2.5em}
\includegraphics[width=\linewidth]{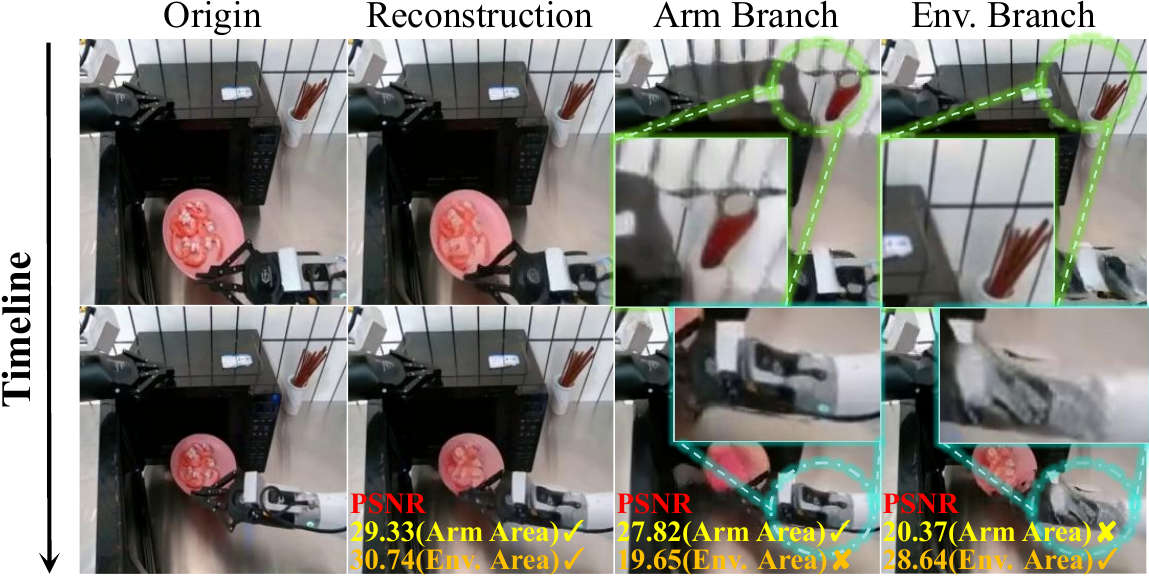}
\caption{
Disentanglement analysis on the Agibot-2025~\cite{jiang2025enerverse} dataset. 
}
\vspace{-2em}
\label{fig:disentangle}
\end{wrapfigure}

We visualize the decoded outputs of each branch in Figure~\ref{fig:disentangle}. As observed, the arm branch captures the robotic motion, while the environment branch models the global scene structure, resulting in disentangled latent representations that facilitate efficient and controllable downstream tasks.

\begin{figure}[!t]
\centering
\vspace{-0.5em}
\centering
\includegraphics[width=\linewidth]{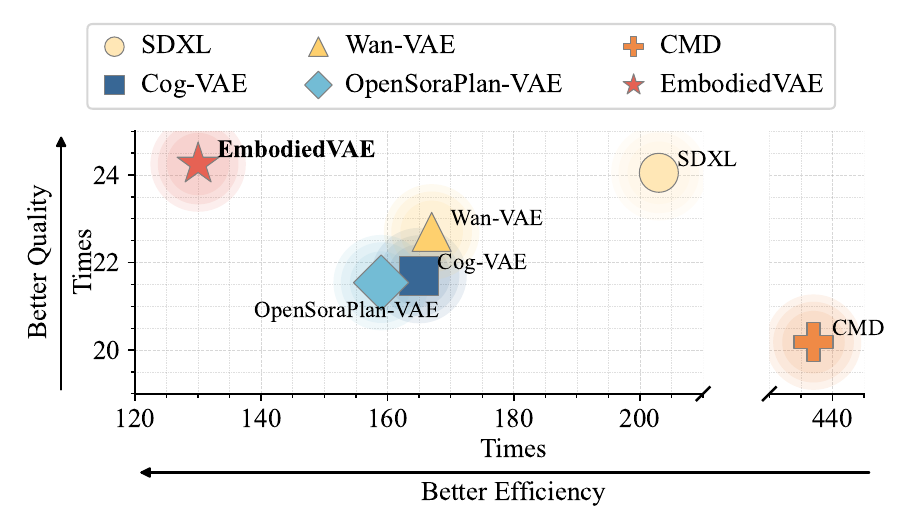}
\captionof{figure}{
Comparison of our proposed \modelname{} and baselines on downstream training efficiency using the Bridge~\cite{walke2023bridgedata} dataset.
The reported time corresponds to 1000 training steps with pre-encoded latents~\cite{yang2024cogvideox} from 256 × 256, 16-frame  videos(17 for Wan-VAE due to its characteristics).
}
\label{fig:efficiency}
\vspace{-1em}
\end{figure}

\subsection{\textit{\underline{Efficiency Studies}}: Downstream LDM Training Time Analysis}\label{sec:exp_efficiency}
To investigate the efficiency benefits of the compact latents introduced by \modelname{} in downstream LDM training, we measure the training time and memory consumption of our model in the downstream robotic manipulation world model, also using IRASim-L~\cite{zhu2025irasim} as an example. The results are reported in Figure~\ref{fig:efficiency}. 
The reported values are measured after the pre-encoding operation, where the latents of the training data are precomputed following previous works~\cite{yang2024cogvideox, zhu2025irasim}.
As observed, the compact latents produced by our \modelname{} effectively reduce resource consumption during downstream training.

\section{Conclusion}
We present \modelname{}, a high-compression and controllable video VAE tailored for embodied manipulation. To enable efficient and high-quality downstream action control, \modelname{} disentangles videos into two complementary latents via a dual-encoder, single-decoder architecture with asymmetric compression. An OT-based motion consistency module further encourages coherent robotic motion preservation. The resulting arm and environment latents enable precise condition injection. Extensive experiments on video reconstruction and action-conditioned manipulation show the effectiveness of \modelname{}.
\clearpage

\section*{Acknowledgements}
The corresponding authors are Jianxin Li and Zhibo Chen. 
This work was supported by the National Natural Science Foundation of China under Grant No. 62225202, and the Zhongguancun Academy Project under Grant No. C20250302.

\bibliographystyle{splncs04}
\bibliography{main}

\clearpage
\setcounter{page}{1}

\appendix
\setcounter{table}{0}
\setcounter{footnote}{0}
\setcounter{figure}{0}
\setcounter{equation}{0}
\setcounter{proposition}{0}

\section{Proof}
\label{sec:proof}
To demonstrate the proposed Proposition~\ref{prop1}, we first introduce the Sinkhorn–Knopp algorithm.

\subsection{Preliminary: Sinkhorn-Knopp Algorithm}

The Sinkhorn-Knopp algorithm is a classical method for efficiently solving discrete optimization problems, and is particularly well-suited for computing the optimal transport (OT) distance between two discrete probability distributions.
In the context of OT, \cite{cuturi2013sinkhorn} introduced an entropy-regularized formulation that enables fast and stable computation. Specifically, entropy regularization smooths the feasible domain of the original OT problem, transforming it into a differentiable matrix-scaling (or permutation) problem. The resulting objective can then be efficiently optimized via the Sinkhorn iterative normalization procedure, which yields an approximate yet accurate solution to the original OT objective.
Formally, the entropy-regularized OT problem reformulates as:
\begin{equation}\label{eq_appendix:ot_discrete_definition}
\begin{aligned}
\mathbf{T}^{\star} &= \arg\min_{\mathbf{T}} \sum_{ij}\mathbf{T}_{ij} c(i,j) + \epsilon \mathcal{H}(\mathbf{T}), \\
& \text{s.t.} \quad \mathbf{T}\mathbf{1} = \boldsymbol{\mu}_{p}, \quad \mathbf{T}^{\top}\mathbf{1} = \boldsymbol{\mu}_{q},
\end{aligned}
\end{equation}
where $\mathbf{T}$ denotes the transport matrix between the source distribution $\boldsymbol{\mu}_{p}$ and the target distribution $\boldsymbol{\mu}_{q}$, $c(i,j)$ represents the transport cost between elements $i$ and $j$, $\epsilon$ is the entropy regularization coefficient, and $\mathcal{H}(\mathbf{T})$ denotes entropy term encouraging smoothness in the transport plan.
From the perspective of probability and information theory, a uniform distribution has the maximum entropy. 
Therefore, entropy regularization will make the optimal transfer matrix $\mathbf{T}^{\star}$ closer to a uniform distribution.
Then Eq.~\eqref{eq_appendix:ot_discrete_definition} can be solved by adopting the Sinkhorn-Knopp Algorithm.

\textbf{Sinkhorn iterative solution.} 
Suppose we have two marginal distributions represented by vectors $\boldsymbol{a}_{u}\in \mathbb{R}^{d_{1}}$ and $\boldsymbol{a}_{v} \in \mathbb{R}^{d_{2}}$.
We initialize them as: 
\begin{equation}
   \boldsymbol{a}_{u}^{(0)}=\frac{\mathbf{1}_{d_{1}}}{d_{1}}, \boldsymbol{a}_{v}^{(0)}=\frac{\mathbf{1}_{d_{2}}}{d_{2}},
\end{equation}
where $\mathbf{1}_{d}$ denotes a $d$-dimensional all-ones vector.
The Sinkhorn iteration alternately updates the scaling vectors as:
\begin{equation}
    \boldsymbol{a}_{u}^{k+1}\leftarrow \frac{\boldsymbol{\mu}_{u}}{\mathbf{K}\cdot\boldsymbol{a}_{u}^{k}},\boldsymbol{a}_{v}^{k+1}\leftarrow \frac{\boldsymbol{\mu}_{v}}{\mathbf{K}^{\top}\cdot\boldsymbol{a}_{u}^{k}},
\end{equation}
where the kernel matrix $\mathbf{K}$ is defined as $\mathbf{K}_{ij}=e^{-\lambda\mathbf{C}_{ij}}$, and $\mathbf{C}$ is the transport cost matrix with entries $\mathbf{C}_{ij}$ measuring the pairwise cost between the 
$i$-th element of the source distribution and the $j$-th element of the target distribution.
After convergence, the optimal transport plan is obtained as: 
\begin{equation}\label{eq_appendix:ot_plan}
    \mathbf{T}^{\star}=\text{diag}(\boldsymbol{a}_{u})\mathbf{K}\text{diag}(\mathbf{a}_{v}),
\end{equation}
and the corresponding minimal transport distance (i.e., the optimal transport cost) is:
\begin{equation}
    \mathbf{D}^{\star}=\langle \mathbf{T}^{\star}, \mathbf{C} \rangle,
\end{equation}
where $\langle\cdot, \cdot\rangle$ denotes the Frobenius inner product, i.e., $\langle\mathbf{T^{\star}, \mathbf{C}}\rangle=\sum_{ij}\mathbf{T}^{\star}_{ij}\mathbf{C}_{ij}$.

\subsection{Proof of Proposition~\ref{prop1}}
Here, we first restate the proposition as follows:
\begin{proposition}
    Let $\mathbf{\hat{T}}^{\star}$ denote the OT plan obtained by the Sinkhorn–Knopp algorithm after row normalization. Denote  $j_i^{\star}=\arg\max_{j}\mathbf{S}_{ij}$ and define the margin $\gamma_i=\mathbf{S}_{i j_i^{\star}}-\max_{j\neq j_i^{\star}}\mathbf{S}_{ij}$, we have $\mathbf{\hat{T}}^{\star}_{ij^{\star}}\ge\frac{1}{1+ae^{-\gamma_{i}/\epsilon}}$ with a constant $a$.
\end{proposition}

\begin{proof}
For the row-normalized OT plan $\hat{\mathbf{T}}^{\star}$, we have $\sum_{j}\hat{\mathbf{T}}^{\star}_{ij}=1$. 
According to Eq.~\eqref{eq_appendix:ot_plan}, the following holds:
\begin{align}
\hat{\mathbf{T}}^{\star}_{i j_i^{\star}}
&= \frac{\hat{\mathbf{T}}^{\star}_{i j_i^{\star}}}{\sum_{j}\hat{\mathbf{T}}^{\star}_{ij}} \\
&= \frac{(\boldsymbol{\mu}_{v})_{j_i^{\star}} e^{\mathbf{S}_{i j_i^{\star}} / \varepsilon}}
{\sum_{j} (\boldsymbol{\mu}_{v})_{j} e^{\mathbf{S}_{ij} / \varepsilon}} \\
&= \frac{1}{
1 + \sum_{j \neq j_i^{\star}}
\frac{(\boldsymbol{\mu}_{v})_{j}}{(\boldsymbol{\mu}_{v})_{j_i^{\star}}}
e^{-(\mathbf{S}_{i j_i^{\star}} - \mathbf{S}_{ij}) / \varepsilon}}
\label{eq_appendix:proof_1} \\
&\ge \frac{1}{1 + a_i e^{-\gamma_i / \varepsilon}},
\label{eq_appendix:proof_2}
\end{align}
where the inequality in Eq.~\eqref{eq_appendix:proof_2} follows from the fact that 
for any $j \neq j_i^{\star}$, we have 
$\mathbf{S}_{i j_i^{\star}} - \mathbf{S}_{ij} \ge \gamma_i$, 
and therefore 
$e^{-(\mathbf{S}_{i j_i^{\star}} - \mathbf{S}_{ij}) / \varepsilon} 
\le e^{-\gamma_i / \varepsilon}$.
Here, $a_i = \sum_{j \neq j_i^{\star}}
\frac{(\boldsymbol{\mu}_{v})_{j}}{(\boldsymbol{\mu}_{v})_{j_i^{\star}}}$.

This completes the proof.
\end{proof}

\section{Compression Rate Computation}
We define the compression rate of visual autoencoders in the same manner as in previous work~\cite{wang2025vidtwin,liu2025hi,wu2025h3ae}. The formal definition is given as follows:
\begin{definition}[Compression Rate of VAEs]
\label{def:compression_rate}
Let $\mathbf{x}\in\mathbb{R}^{H\times W\times T\times C}$ denote the input video with spatial resolution $(H,W)$, temporal length $T$, and channel dimension $C$. 
Let $\mathbf{z}\in\mathbb{R}^{H'\times W'\times T'\times C'}$ denote the corresponding latent representation produced by the VAE encoder. 
The \emph{compression rate} $\mathcal{R}$ is defined as the ratio between the total number of latent variables and that of the input video, i.e.,
\begin{equation}
r= \frac{H'W'T'C'}{HWT C}.
\end{equation}
\end{definition}
A smaller $r$ indicates a higher compression level, implying that fewer latent tokens are used to represent the same visual information, thereby making the training of downstream video generative models more efficient.
The compression rate of the proposed \modelname{} and all baseline models is computed as follows:

\begin{itemize}
    \item \textbf{Compression rate of \modelname{} (ours)}: 
    The proposed \modelname{} employs a dual-encoder architecture that produces two complementary latent representations. 
    Specifically, the \emph{robotic-arm encoder} applies a temporal downsampling factor of $4$ and a spatial downsampling factor of $16$, 
    while the \emph{environment-context encoder} uses a temporal downsampling factor of $8$ and a spatial downsampling factor of $8$. 
    Both latent spaces have a channel dimension of $4$, leading to an overall compression rate of:
    \begin{equation}
    \begin{aligned}
        &r_{\text{EmbodiedVAE}} \\
         =& \frac{4}{3\times(4\times16\times16)} + \frac{4}{3\times(8\times8\times8)} \\
        \approx& 0.39\%.
    \end{aligned}
    \end{equation}

    \item \textbf{Compression rate of OpenSoraPlan-VAE}~\cite{lin2024open}: OpenSoraPlan-VAE applies a temporal downsampling factor of $4$ and a spatial downsampling factor of $8$, with the latent channel dimension of $8$ resulting in a compression rate of:
    \begin{equation}
        r_{\text{OpenSoraPlan-VAE}} = \frac{8}{3 \times (4 \times 8 \times 8)} \approx 1.04\%.
    \end{equation}

    \item \textbf{Compression rate of Cosmos}~\cite{agarwal2025cosmos}: 
    We adopt the \textit{Cosmos-Tokenizer}, specifically with ``CV4x8x8'', which applies a temporal downsampling factor of $4$ and a spatial downsampling factor of $8$, 
    with a latent channel dimension of $16$, resulting in a compression rate of:
    \begin{equation}
        r_{\text{Cosmos}} = \frac{16}{3 \times (4 \times 8 \times 8)} \approx 2.08\%.
    \end{equation}

    \item \textbf{Compression rate of iVideoGPT}~\cite{wu2024ivideogpt}:
    iVideoGPT adopts a distinct compression strategy, where the first $t_{0}$ context frames are encoded using $n_{0}$ tokens, 
    while the subsequent frames are represented with fewer tokens $n_{1} < n_{0}$. 
    Based on the configuration reported in the original paper, the overall compression rate is computed as:
    \begin{equation}
    \begin{aligned}
         r_{\text{iVideoGPT}} 
        &= \frac{2 \times 16^{2} \times 64 + 14 \times 4^{2} \times 64}
        {3 \times 16 \times 256^{2}} \\
        &\approx 1.50\%.
    \end{aligned}
    \end{equation}

    \item \textbf{Compression rate of MAGVIT-v2}~\cite{yu2023language}:
    MAGVIT-v2 applies a temporal downsampling factor of $4$ and a spatial downsampling factor of $8$. 
    The latent channel dimension is set to $5$, as reported in the original paper, resulting in a compression rate of:
    \begin{equation}
        r_{\text{MAGVIT-v2}} = \frac{5}{3 \times (4 \times 8 \times 8)} \approx 0.65\%.
    \end{equation}

    \item \textbf{Compression rate of VidTwin}~\cite{wang2025vidtwin}: 
    Vidtwin decouples the input video into two latent representations of sizes $7 \times 7 \times 16 \times 4$ and $16 \times 14 \times 8$, 
    for an input video clip of size $16 \times 3 \times 224 \times 224$. 
    Accordingly, the overall compression rate is computed as:
    \begin{equation}
    \begin{aligned}
         r_{\text{VidTwin}} 
        &= \frac{7 \times 7 \times 16 \times 4 + 16 \times 14 \times 8}
        {3 \times 16 \times 224 \times 224} \\
        &\approx 0.20\%.
    \end{aligned}
    \end{equation}
    
    \item \textbf{Compression rate of EMU-3}~\cite{wang2024emu3}: 
    Similar to MAGVIT-v2, EMU-3 achieves a $4\times$ temporal compression and an $8\times8$ spatial compression. 
    With a latent channel dimension of $4$, its compression rate is computed in the same manner as:
    \begin{equation}
        r_{\text{EMU-3}} 
        = \frac{4}{3 \times (4 \times 8 \times 8)} 
        \approx 0.53\%.
    \end{equation}

    \item \textbf{Compression rate of CV-VAE}~\cite{zhao2024cv}: 
    In terms of compression rate, CV-VAE matches EMU-3, 
    achieving a $4\times$ temporal compression and an $8\times8$ spatial compression. 
    Accordingly, its compression rate is identical to EMU-3:
    \begin{equation}
        r_{\text{CV-VAE}} = \frac{4}{3 \times (4 \times 8 \times 8)} \approx 0.53\%.
    \end{equation}
    
    \item \textbf{Compression rate of CMD}~\cite{yu2024efficient}: 
    CMD decouples video representations into \emph{content frames} and \emph{motion latents}. 
    For a video of size $(C, F, H, W)$, the content frame has dimensions $(C, H, W)$, 
    while the motion latent has dimensions $(D, H + W, F)$, where $D$ denotes the dimension of the motion vector. 
    Based on the configuration reported in the paper, the compression rate is computed as:
    \begin{equation}
    \begin{aligned}
        r_{\text{CMD}} 
        &= \frac{1}{F} + \frac{D\times(H + W)}{C\times H\times W} \\
        &= \frac{1}{16} + \frac{2 \times 224 \times 32}{3 \times 224 \times 224} 
        \approx 6.85\%.
    \end{aligned}
    \end{equation}

    \item \textbf{Compression rate of SDXL}~\cite{podell2023sdxl}: 
    SDXL is an image-based VAE that encodes videos frame by frame without any temporal compression. 
    Each frame is encoded with an $8\times8$ spatial downsampling factor and a latent channel dimension of $16$. 
    Thus, the overall compression rate is computed as:
    \begin{equation}
        r_{\text{SDXL}} 
        = \frac{16}{3 \times 8 \times 8} 
        \approx 8.32\%.
    \end{equation}
        
    \item \textbf{Compression rate of Cog-VAE}~\cite{yang2024cogvideox}: 
    Cog-VAE applies a $4\times$ temporal downsampling and an $8\times8$ spatial downsampling, 
    resulting in a latent representation with a channel dimension of $16$. 
    Accordingly, the overall compression rate is computed as:
    \begin{equation}
        r_{\text{Cog-VAE}} 
        = \frac{16}{3 \times (4 \times 8 \times 8)} 
        \approx 2.08\%.
    \end{equation}

    \item \textbf{Compression rate of Wan-VAE}~\cite{wan2025wan}: 
    Wan-VAE, similar to Cog-VAE, applies a $4\times$ temporal downsampling and an $8\times8$ spatial downsampling, 
    resulting in a latent representation with a channel dimension of $16$. 
    However, it encodes and decodes the first frame independently. 
    For a video of $4T+1$ frames, the model produces $T+1$ latent representations. 
    Taking a $3\times17\times256\times256$ video clip as an example, the compression rate is computed as:
    \begin{equation}
        r_{\text{Wan-VAE}} 
        = \frac{16 \times 5 \times 32 \times 32}{3 \times 17 \times 256 \times 256} 
        \approx 2.45\%.
    \end{equation}
   
\end{itemize}

\section{Experiment Details}\label{sec:exp}
\subsection{Details about the VAE Reconstruction}
\subsubsection{Implement Details}
We adopt a self-curated dataset containing approximately one million videos, 
each accompanied by pre-generated robotic arm masks obtained using 
SAM2~\cite{ravi2024sam} and Grounding DINO~\cite{ren2024grounding} 
with the text prompt ``Robotic Arm''. 
Our model is trained on 3~fps, 16-frame, $256\times256$ video clips 
and evaluated on 16~fps, 16-frame, $256\times256$ video clips. 
Training is conducted on 16 NVIDIA A800 GPUs (40~GB). 
The detailed training configuration is provided in Table~\ref{tb:configuration}, where $\lambda_{\text{perception}}$ denotes the weight of the perception loss~\cite{tang2024vidtok}, 
and $\lambda_{\text{GAN}}$ denotes the weight of the standard GAN loss, 
while $\beta$ represents the weight of the KL divergence term 
and $\alpha$ corresponds to the weight of the OT-based consistency loss.

\subsubsection{Baseline Details}
For the baselines including OpenSora~\cite{lin2024open}, 
Cosmos~\cite{agarwal2025cosmos}, \\ iVideoGPT~\cite{wu2024ivideogpt}, 
VidTwin~\cite{wang2025vidtwin}, EMU-3~\cite{wang2024emu3}, 
CV-VAE~\cite{zhao2024cv}, Cog-VAE~\cite{yang2024cogvideox}, 
and Wan-VAE~\cite{wan2025wan}, 
we adopt the official checkpoints and implementations provided in their repositories. 
For MAGVIT-v2~\cite{yu2023language} and CMD~\cite{yu2024efficient}, 
since official code and pretrained checkpoints are not publicly available, 
we re-implemented these methods based on the descriptions in their respective papers.

\begin{table}[!htbp]
  \centering
  \caption{Training Configuration.}
  \vspace{-0.5em}
  \label{tb:configuration}
  \tabcolsep=0.1cm
  \resizebox{0.8\linewidth}{!}{ 
   \begin{tabular}{L{5cm}|C{5cm}}
    \toprule
    \multicolumn{1}{c|}{\textbf{Parameter }} & \textbf{Value} \\
    \midrule
    Input Video Resolution & $256\times 256$ \\
    Input Video Frames & 16 \\
    Input Video FPS & 3 \\
    Adam Optimizer  & $\beta_{1}=0.9$, $\beta_{2}=0.99$ \\
    \midrule
    \multicolumn{2}{c}{\textit{First Training Stage}} \\
    \midrule
    Learning Rate & $1.0 \times 10^{-5}$  \\
    wegiht decay & $1.0 \times 10^{-4}$ \\
    $\lambda_{\text{perception}}$ & $1.0 \times 10^{0}$ \\
    $\lambda_{\text{GAN}}$ & $5.0\times 10^{-3}$ \\
    $\beta$ & $1.0\times 10^{-6}$ \\
    $\alpha$ & $1.0\times 10 ^{4}$ \\
    \midrule
    \multicolumn{2}{c}{\textit{Second Training Stage}} \\
    \midrule
    Learning Rate (Decoder) & $1.0 \times 10^{-5}$ \\
    Learning Rate (Encoder) & $1.0 \times 10^{-8}$ \\
    $\lambda_{\text{perception}}$ & $1.0 \times 10^{0}$ \\
    $\lambda_{\text{GAN}}$ & $5.0\times 10^{-3}$ \\
    $\beta$ & $1.0\times 10^{-6}$ \\
    \bottomrule
    \end{tabular}%
}
\end{table}

\subsection{Details about Action-Controlled Video Prediction for Robotic Manipulation}

\subsubsection{Task Description}  
Video generative models have recently demonstrated strong potential as world models. 
When training such models for robotic manipulation, it is crucial to accurately simulate the complex interactions among the robot, manipulated objects, and the surrounding environment. 
In this task, we evaluate whether the models can generate realistic videos that depict fine-grained robot–object interaction dynamics, given a sequence of historical observations and a corresponding action trajectory.

\subsubsection{LDM Backbone}
To ensure fairness, we adopt IRASim-L (461M)~\cite{zhu2025irasim} as the downstream LDM backbone for all robotic manipulation video prediction experiments, 
while the only difference lies in the choice of VAEs (including our proposed \modelname{} and the baseline models) used during training and inference.
IRASim is a DiT-style model that incorporates a novel frame-level action-conditioning module within each transformer block, 
explicitly modeling and strengthening the alignment between every action and its corresponding video frame.
IRASim provides four model scales: 
(1) \textit{IRASim-S} with 12 layers and a hidden size of 768, 
(2) \textit{IRASim-B} with 12 layers and a hidden size of 768, 
(3) \textit{IRASim-L} with 24 layers and a hidden size of 1024, 
and (4) \textit{IRASim-XL} with 28 layers and a hidden size of 1152. 
To balance performance and computational resource consumption, we adopt IRASim-L, which contains 461M parameters, for all robotic manipulation experiments.

\subsubsection{Implement Details}
\noindent\textbf{Action Condition Injection}
As different VAEs exhibit varying temporal compression rates, we detail the action-injection strategies applied to each VAE’s latent representations for IRASim as follows:
\begin{figure*}[!t]
\centering
\includegraphics[width=\linewidth]{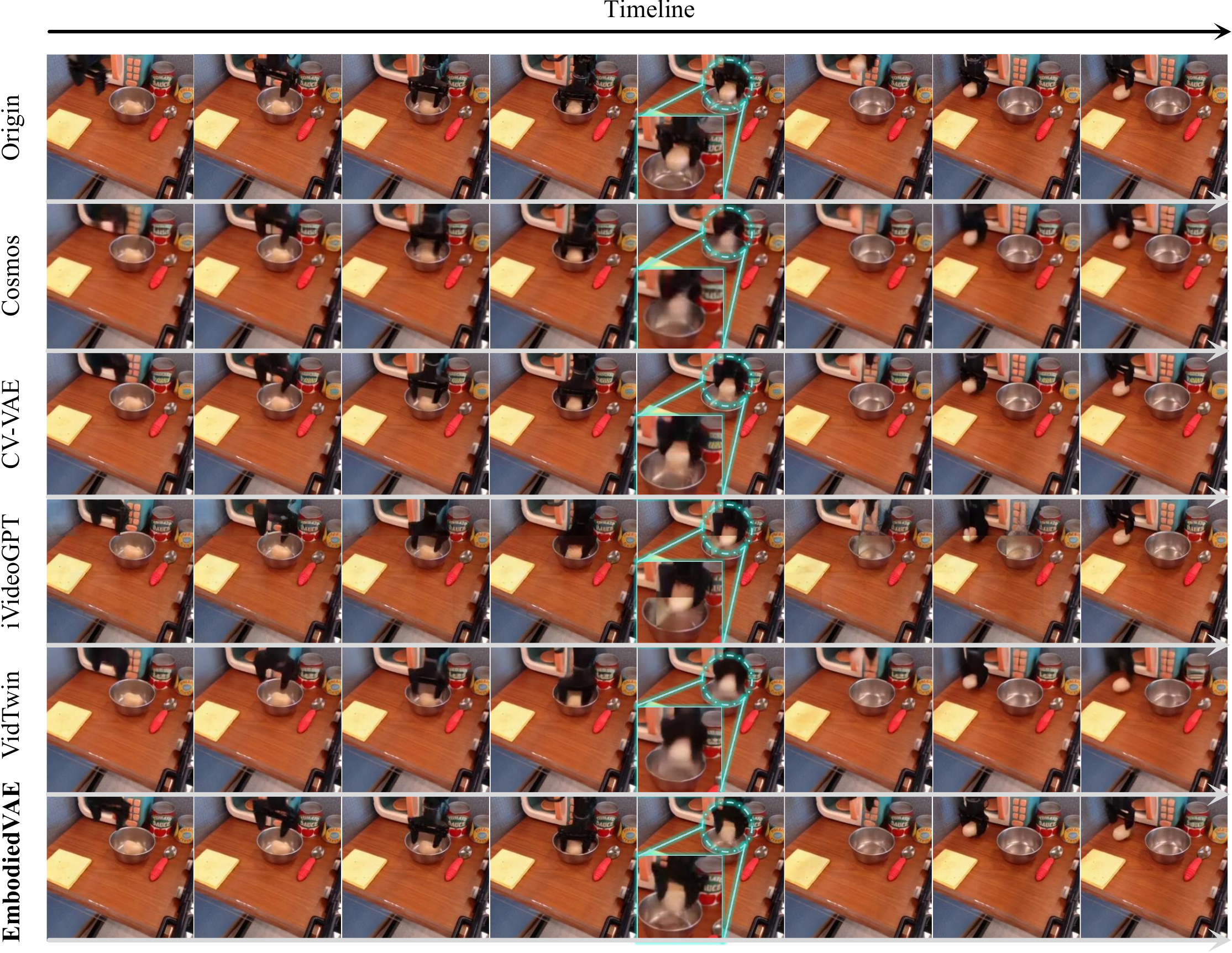}
\caption{Additional examples of VAE reconstruction tasks on the Bridge~\cite{walke2023bridgedata} dataset.}
\label{appendix_fig:recon_example_bridge}
\end{figure*}
\begin{figure*}[!t]
\centering
\includegraphics[width=\linewidth]{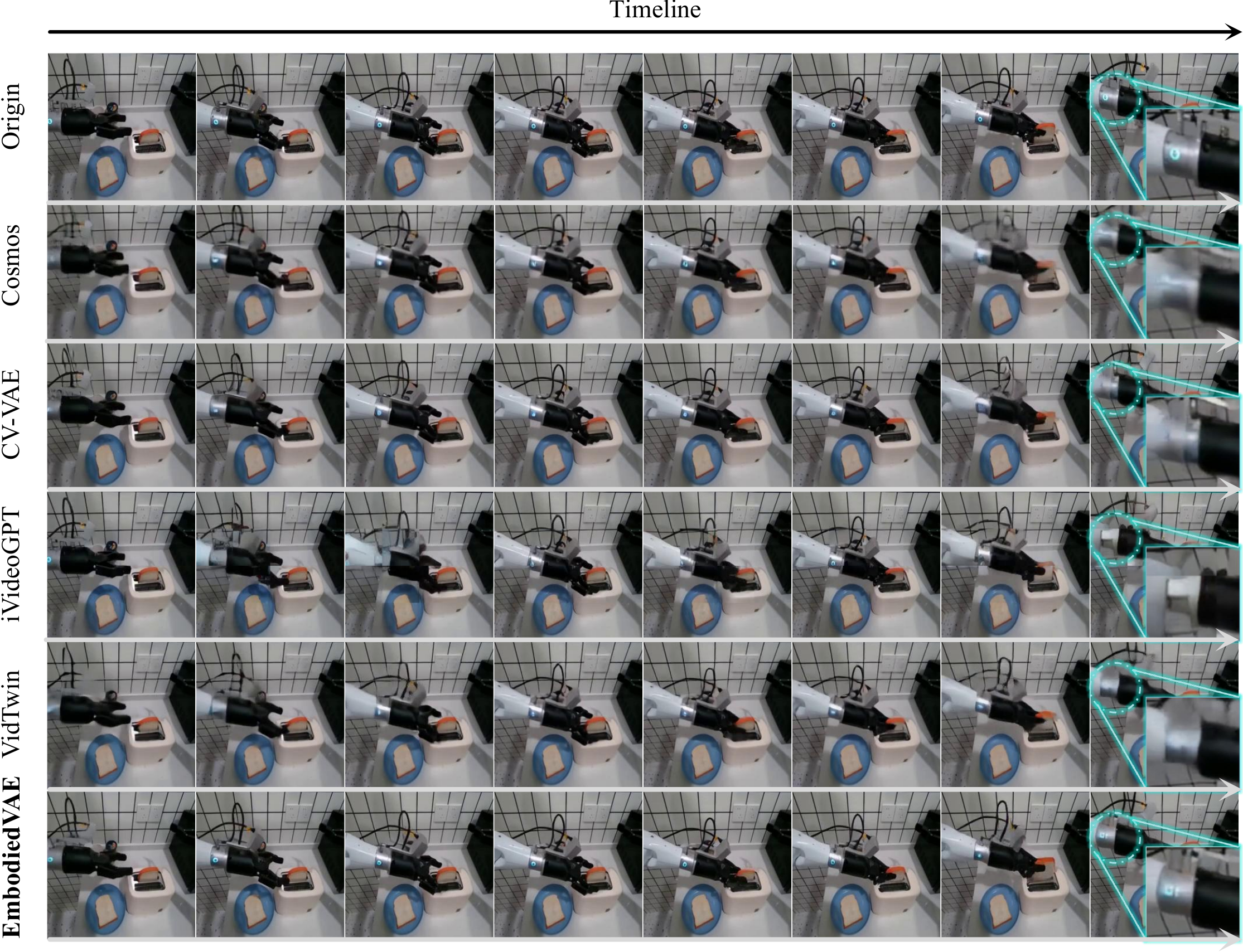}
\caption{Additional examples of VAE reconstruction tasks on the Agibot-2025~\cite{bu2025agibot} dataset.}
\label{appendix_fig:recon_example_agibot}
\end{figure*}
\begin{itemize}
    \item For our proposed \textit{\modelname{}}, given a sequence of $T\!-\!1$ action vectors 
    $\mathbf{A}\in\mathbb{R}^{(T-1)\times d}$, 
    where $d$ denotes the action dimension, each action represents the relative motion with respect to the first frame. 
    We first pad a zero vector to align the sequence with the $T$ video frames, 
    obtaining $\hat{\mathbf{A}}\in\mathbb{R}^{T\times d}$. 
    A learnable 1D convolution is then applied to temporally downsample $\hat{\mathbf{A}}$ 
    into $\hat{\mathbf{A}}_{\text{arm}}\in\mathbb{R}^{\frac{T}{4}\times d}$, 
    which is injected into the latent representation of the robotic arm, $\mathbf{Z}_{\text{arm}}$. 
    Subsequently, another 1D convolution is used to further downsample 
    $\hat{\mathbf{A}}_{\text{arm}}$ into 
    $\hat{\mathbf{A}}_{\text{env}}\in\mathbb{R}^{\frac{T}{8}\times d}$, 
    representing the environmental actions that are injected into the latent representation 
    $\mathbf{Z}_{\text{env}}$.
    Both latent representations share the same diffusion process and LDM parameters.

    \item For the \textit{OpenSoraPlan-VAE}~\cite{lin2024open}, \textit{Emu3}~\cite{wang2024emu3}, \textit{CV-VAE}~\cite{zhao2024cv}, and \textit{Cog-VAE}~\cite{yang2024cogvideox} which resulting the latent with 4 temporal downsamlping, the same as our \modelname{} we adopting the 1D convolution to downsampleing the $\mathbf{\hat{A}}$ into $\mathbf{\hat{A}}_{\text{align}}\in\mathbb{R}^{\frac{T}{4}\times d}$ to provide aligned actions. 

    \item As for \textit{Wan-VAE}~\cite{wan2025wan}, since it encodes the first frame independently, 
    we follow its encoding strategy when injecting action information. 
    Given $T\!+\!1$ frames, we apply a 1D convolution to the actions of the last $T$ frames, 
    downsampling them into $\tfrac{T}{4}$ action conditions. 
    These are then concatenated with the first-frame action, 
    yielding the aligned action sequence 
    $\mathbf{\hat{A}}_{\text{align}} \in \mathbb{R}^{(1+\frac{T}{4})\times d}$.

    \item For \textit{CMD}~\cite{yu2024efficient}, since it decouples the input into a content latent without a temporal dimension 
    and a motion latent without temporal compression, 
    we follow the original paper and employ two independent diffusion processes to denoise the motion and content latents, 
    each modeled by an IRASim DiT architecture of the same design. 
    We apply the complete action sequence $\mathbf{\hat{A}}\in \mathbb{R}^{T\times d}$ to the motion diffusion branch, 
    as the motion latent preserves the full temporal resolution. 
    For the content diffusion branch, we use the content latent representations of the historical frames as conditional inputs 
    to guide the generation.

    \item For \textit{\textit{VidTwin}}~\cite{wang2025vidtwin}, 
    it produces three latent representations: 
    a structure latent 
    $\mathbf{Z}_{S}\in\mathbb{R}^{d_{S}\times T\times h_{S}\times w_{S}}$ 
    and two dynamics latents, 
    $\mathbf{Z}_{D}^{h}\in\mathbb{R}^{d_{D}\times T\times h_{D}}$ and 
    $\mathbf{Z}_{D}^{w}\in\mathbb{R}^{d_{D}\times T\times w_{D}}$. 
    Following the procedure described in the original paper, 
    we apply two independent patchification modules to convert both dynamics latents into sequences of tokens. 
    These token sequences are normalized and aligned to a comparable scale, 
    then concatenated along the temporal dimension. 
    Finally, the complete action sequence $\hat{\mathbf{A}}$ is injected into the latent representations.

    \item For \textit{SDXL}\cite{podell2023sdxl}, the official VAE used in IRASim\cite{zhu2025irasim}, we follow its default configuration and represent actions as relative motion with respect to the first frame.
\end{itemize}

\noindent In this way, we ensure fairness by injecting the action conditions consistently across all models. 
For CMD, we adopt two independent diffusion processes 
to denoise its motion and content latents separately, 
following the official implementation~\cite{yu2024efficient}. 
In contrast, a single diffusion process is employed in all other models 
Each diffusion process consists of 1000 steps, 
with DDIM~\cite{song2020denoising} used as the sampling strategy 
and 50 steps applied during inference.

\begin{figure*}[!t]
\centering
\includegraphics[width=0.5\linewidth]{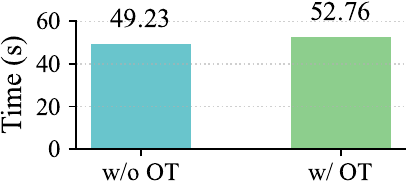}
\caption{Computational Time Analysis of the OT-based Module.}
\label{appendix_fig:time}
\end{figure*}

\section{Computational Time Analysis of the OT-based Module}
Although the complexity of this module is $\mathcal{O}(\Delta TN^{2}d)$, in practice, the OT-based loss is applied to latents with the input resolution $256\times256$, and our robotic arm encoder applies strong spatial compression with $16\times 16$, so the actual computational overhead is modest. 
We report the training time per 100 samples with and without motion loss on A100 GPUs in Figure~\ref{appendix_fig:time}, showing that \textit{the cost introduced by the OT module is acceptable in practice}.

\begin{figure*}[!htbp]
\centering
\includegraphics[width=1.0\linewidth]{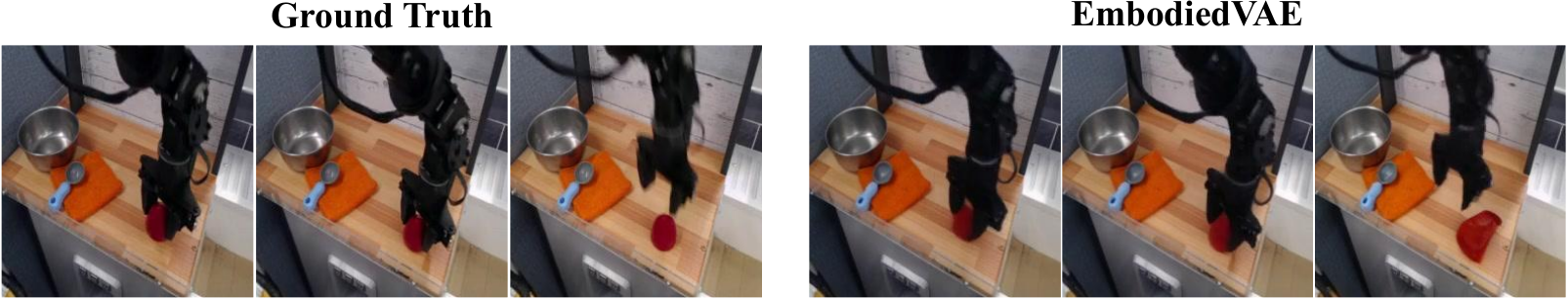}
\caption{Scenarios Where Robotic Arms Occupy Large Spatial Regions.}
\label{appendix_fig:assumption}
\end{figure*}

\section{Scenarios Where Robotic Arms Occupy Large Spatial Regions}
We provide examples in Figure~\ref{appendix_fig:assumption} where the robotic arms occupy relative large portion of the frame. As shown, \textit{even when the robotic arms violate the assumption} of occupying only a small spatial region, our proposed model still achieves high-quality reconstruction and manipulation performance.

\begin{figure*}[!t]
\centering
\includegraphics[width=\linewidth]{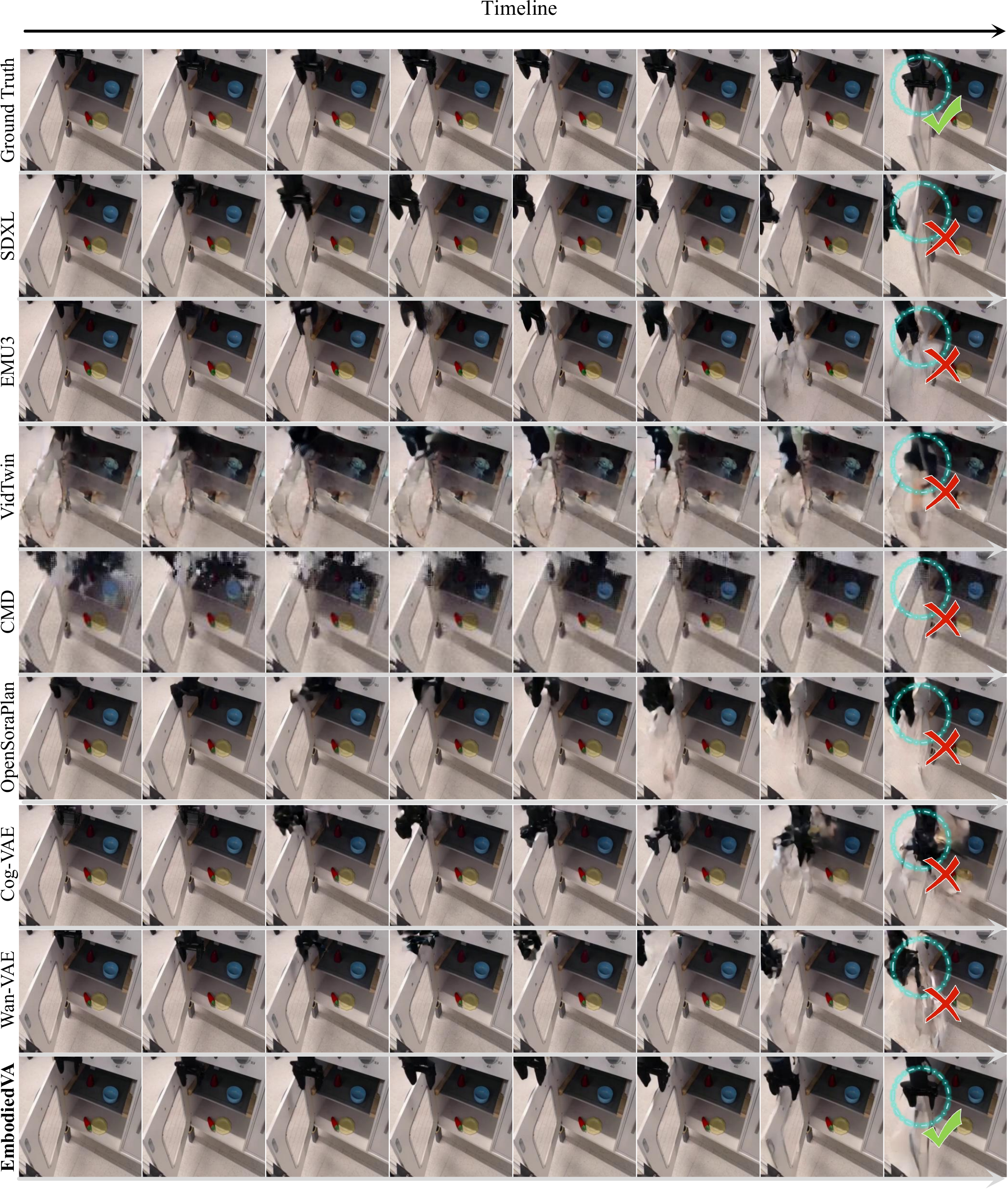}
\caption{Additional examples of video prediction tasks in robotic manipulation scenarios on the Bridge~\cite{walke2023bridgedata} dataset.}
\label{appendix_fig:gen_example_bridge}
\end{figure*}

\begin{figure*}[!t]
\centering
\includegraphics[width=\linewidth]{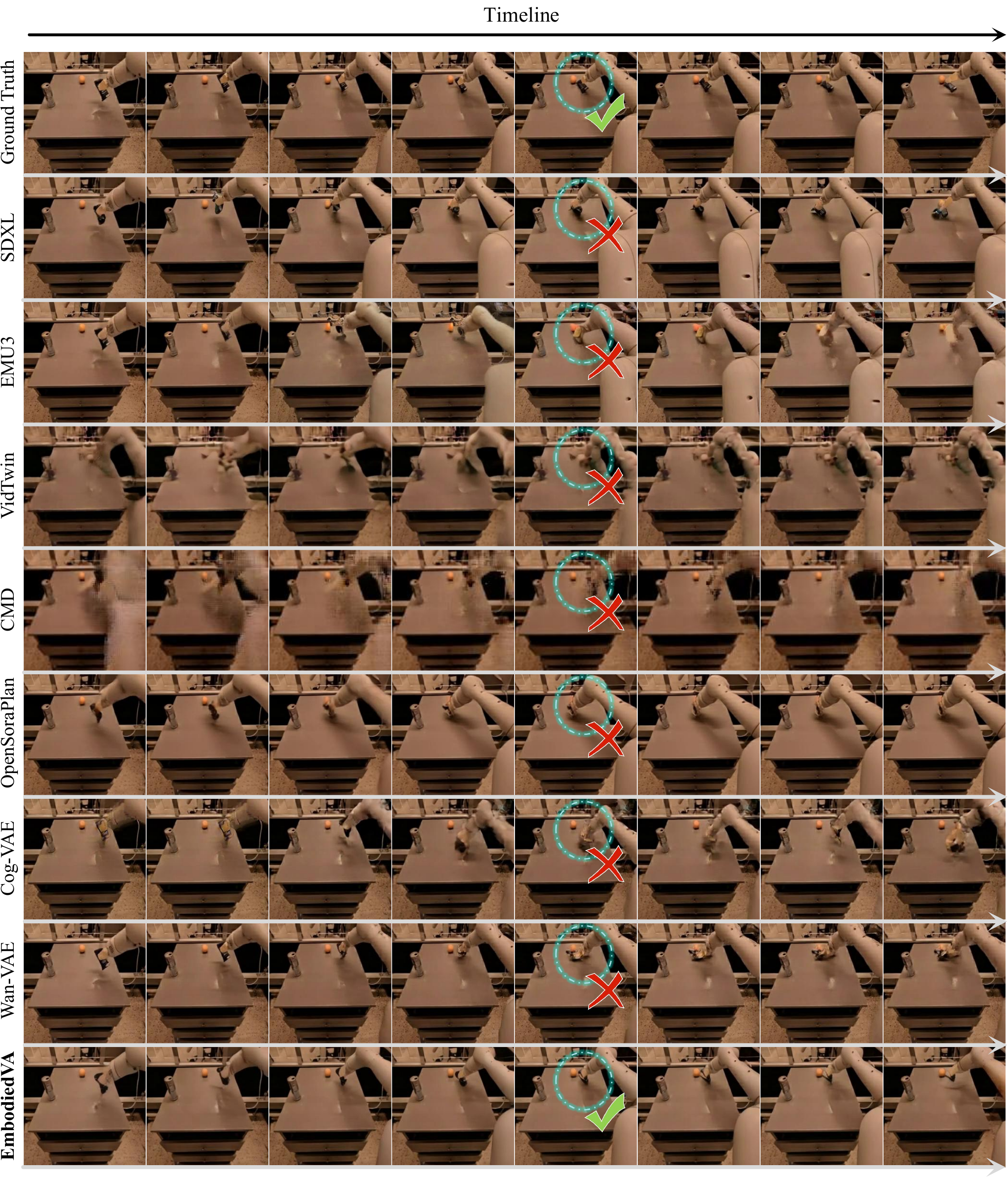}
\caption{Additional examples of video prediction tasks in robotic manipulation scenarios on the RT-1~\cite{brohan2022rt} dataset.}
\label{appendix_fig:gen_example_rt1}
\end{figure*}

\section{Additional Results}
Here, we present additional experimental results for both tasks.
The additional experimental results for the video VAE reconstruction tasks are shown in Figure~\ref{appendix_fig:recon_example_bridge} and Figure~\ref{appendix_fig:recon_example_agibot}. As can be observed, our proposed \modelname{} demonstrates superior reconstruction capability, particularly in accurately capturing the robotic arms and electronic components during manipulation.
Furthermore, the additional results for downstream video prediction in robotic manipulation scenarios are presented in Figure~\ref{appendix_fig:gen_example_bridge} and Figure~\ref{appendix_fig:gen_example_rt1}. As shown, \modelname{} provides a more controllable and compact latent representation that enables more accurate action control. This improvement can be attributed to our disentangling modules coupled with the OT-based motion consistency module, which together produce explicit, temporally consistent motion latent representations that are free from the irrelevant background information.

\end{document}